\documentclass[twoside, 11pt]{article}

\usepackage{jmlr2e}
\usepackage{microtype}
\usepackage{subfigure}
\usepackage{booktabs} 
\usepackage{hyperref, centernot}

\usepackage{url}
\usepackage{mathtools}
\mathtoolsset{
showonlyrefs=true 
}
\usepackage{amsmath, bbm, amsfonts, amssymb, cite, enumerate, algorithm, algorithmic}
\usepackage{cancel}
\usepackage{datetime,comment}
\usepackage{centernot}

\DeclareMathOperator*{\sign}{sign}
\DeclareMathOperator*{\algn}{align}

\usepackage{xcolor}

\newcommand{\eod}{{${}$\\}}

\newcommand*\bxi{\ensuremath{\boldsymbol\xi}}
\newcommand*\grad{\nabla}
\newcommand\dd{\partial}

\newcommand{\wt}{{\mathbf w}}
\newcommand{\bp}{{\mathbf p}}
\newcommand{\bq}{{\mathbf q}}

\newcommand{\x}{{\mathbf x}}
\newcommand{\Jac}{{\mathbf J}}

\newcommand{\y}{{\mathbf y}}
\newcommand{\bd}{{\mathbf d}}

\newcommand{\bD}{{\mathbf D}}

\newcommand{\bB}{{\mathbf B}}
\newcommand{\bU}{{\mathbf U}}
\newcommand{\ut}{{\mathbf u}}
\newcommand{\vt}{{\mathbf v}}
\newcommand{\bV}{{\mathbf V}}

\newcommand{\bF}{{\mathbf F}}

\newcommand{\bR}{{\mathbb R}}

\newcommand{\indicator}{{\mathbf I}}

\DeclareRobustCommand\iff{\;\Longleftrightarrow\;}

\newcommand{\al}{\alpha}
\newcommand{\la}{\lambda}
\newcommand{\R}{\mathbb{R}}
\newcommand{\N}{\mathbb{N}}

\newcommand{\cH}{{\mathcal H}}
\newcommand{\bA}{{\mathbf A}}
\newcommand{\bS}{{\mathbf S}}
\newcommand{\bT}{{\mathbf T}}

\newcommand{\bP}{{\mathbf P}}
\newcommand{\bQ}{{\mathbf Q}}
\newcommand{\bM}{{\mathbf M}}
\newcommand{\bC}{{\mathbf C}}

\newcommand{\bbb}{{\mathbf b}}

\newtheorem{thm}{Theorem}
\newtheorem{cor}[thm]{Corollary}
\newtheorem{prop}[thm]{Proposition}

\newtheorem{lem}[thm]{Lemma}
\newtheorem{defn}{Definition}

\newtheorem{rem}{Remark}
\newtheorem{eg}{Example}

\firstpageno{1}

\usepackage{lastpage}
\jmlrheading{20}{2019}{1-\pageref{LastPage}}{1/19}{2/19}{19-008}{Alistair Letcher,
David Balduzzi, 
S{\'e}bastien Racani{\`e}re,
James Martens,
Jakob Foerster,
Karl Tuyls,
Thore Graepel}
\ShortHeadings{Differentiable Game Mechanics}{Letcher, Balduzzi, Racani{\`e}re, Martens, Foerster, Tuyls, Graepel}

\begin{document}

\title{Differentiable Game Mechanics}
\author{\name Alistair Letcher$^*$
        \email ahp.letcher@gmail.com \\
       \addr University of Oxford
       \AND
       \name David Balduzzi$^*$
       \email dbalduzzi@google.com \\
       \addr DeepMind
       \AND 
       \name S{\'e}bastien Racani{\`e}re
       \email sracaniere@google.com \\
       \addr DeepMind
       \AND
       \name James Martens
       \email jamesmartens@google.com\\
       \addr DeepMind
       \AND
       \name Jakob Foerster
       \email jakobfoerster@gmail.com\\
       \addr University of Oxford
       \AND
       \name Karl Tuyls
       \email karltuyls@google.com\\
       \addr DeepMind
       \AND
       \name Thore Graepel
       \email thore@google.com\\
       \addr DeepMind
}

\editor{Kilian Weinberger}

\maketitle

\begin{abstract}
    Deep learning is built on the foundational guarantee that gradient descent on an objective function converges to local minima. Unfortunately, this guarantee fails in settings, such as generative adversarial nets, that exhibit multiple interacting losses. The behavior of gradient-based methods in games is not well understood -- and is becoming increasingly important as adversarial and multi-objective architectures proliferate. In this paper, we develop new tools to understand and control the dynamics in $n$-player differentiable games. 
    
    The key result is to decompose the game Jacobian into two components. The first, symmetric component, is related to potential games, which reduce to gradient descent on an implicit function. The second, antisymmetric component, relates to \emph{Hamiltonian games}, a new class of games that obey a conservation law akin to conservation laws in classical mechanical systems. The decomposition motivates \emph{Symplectic Gradient Adjustment} (SGA), a new algorithm for finding stable fixed points in differentiable games. Basic experiments show SGA is competitive with recently proposed algorithms for finding stable fixed points in GANs -- while at the same time being applicable to, and having guarantees in, much more general cases.
\end{abstract}


\section{Introduction}

A significant fraction of recent progress in machine learning has been based on applying gradient descent to optimize the parameters of neural networks with respect to an objective function. The objective functions are carefully designed to encode particular tasks such as supervised learning. A basic result is that gradient descent converges to a local minimum of the objective function under a broad range of conditions \citep{lee:17}. However, there is a growing set of algorithms that do not optimize a single objective function, including: generative adversarial networks \citep{goodfellow:14, zhu:17}, proximal gradient TD learning \citep{liu:16}, multi-level optimization \citep{pfau:16}, synthetic gradients \citep{jaderberg:17}, hierarchical reinforcement learning \citep{wayne:14,vezhnevets:17}, intrinsic curiosity \citep{pathak:17, burda:19}, and imaginative agents \citep{weber:17}. In effect, the models are trained via games played by cooperating and competing modules.

The time-average of iterates of gradient descent, and other more general no-regret algorithms, are guaranteed to converge to coarse correlated equilibria in games \citep{stoltz:07}. However, the dynamics do not converge to Nash equilibria -- and do not even stabilize in general \citep{mertikopoulos:18, papadimitriou:18}. Concretely, cyclic behaviors emerge even in simple cases, see example~\ref{eg:basic_eg}.

This paper presents an analysis of the second-order structure of game dynamics that allows to identify two classes of games, potential and Hamiltonian, that are easy to solve separately.  We then derive \emph{symplectic gradient adjustment}\footnote{Source code is available at \url{https://github.com/deepmind/symplectic-gradient-adjustment}.} (SGA), a method for finding stable fixed points in games. SGA's performance is evaluated in basic experiments.

\subsection{Background and problem description}
Tractable algorithms that converge to Nash equilibria have been found for restricted classes of games: potential games, two-player zero-sum games, and a few others \citep{hart:13}. Finding Nash equilibria can be reformulated as a nonlinear complementarity problem, but these are `hopelessly impractical to solve' in general \citep{shoham:08} because the problem is PPAD hard \citep{daskalakis:09}.

Players are primarily neural nets in our setting. For computational reasons we restrict to gradient-based methods, even though game-theorists have considered a much broader range of techniques. Losses are not necessarily convex in \emph{any} of their parameters, so Nash equilibria are not guaranteed to exist. Even leaving existence aside, finding Nash equilibria in nonconvex games is analogous to, but much harder than, finding global minima in neural nets -- which is not realistic with gradient-based methods. 

There are at least three problems with gradient-based methods in games. Firstly, the potential existence of cycles (recurrent dynamics) implies there are no convergence guarantees, see example~\ref{eg:basic_eg} below and \citet{mertikopoulos:18}. Secondly, even when gradient descent converges, the rate of convergence may be too slow in practice because `rotational forces' necessitate extremely small learning rates, see figure~\ref{f:learning_rates}. Finally, since there is no single objective, there is no way to measure progress. Concretely, the losses obtained by the generator and the discriminator in GANs are not useful guides to the quality of the images generated. Application-specific proxies have been proposed, for example the inception score for GANs \citep{salimans:16}, but these are of little help during training. The inception score is domain specific and is no substitute for looking at samples. This paper tackles the first two problems.

\subsection{Outline and summary of main contributions}
\paragraph{The infinitesimal structure of games.}
We start with the basic case of a zero-sum bimatrix game: example~\ref{eg:basic_eg}. It turns out that the dynamics under simultaneous gradient descent can be reformulated in terms of Hamilton's equations. The cyclic behavior arises because the dynamics live on the level sets of the Hamiltonian. More directly useful, gradient descent on the Hamiltonian converges to a Nash equilibrium.

Lemma~\ref{t:helmholtz} shows that the Jacobian of any game decomposes into symmetric and antisymmetric components. There are thus two `pure' cases corresponding to when the Jacobian is symmetric and anti-symmetric. The first case, known as potential games \citep{monderer:96}, have been intensively studied in the game-theory literature because they are exactly the games where gradient descent \emph{does} converge.

The second case, Hamiltonian\footnote{\citet{lu:92} defined an unrelated notion of Hamiltonian game.} games, were not studied previously, probably because they coincide with zero-sum games in the bimatrix case (or constant-sum, depending on the constraints). Zero-sum and Hamiltonian games differ when the losses are not bilinear or when there are more than two players. Hamiltonian games are important because (i) they are easy to solve and (ii) general games combine potential-like and Hamiltonian-like dynamics. Unfortunately, the concept of a zero-sum game is too loose to be useful when there are many players: any $n$-player game can be reformulated as a zero-sum $(n+1)$-player game where $\ell_{n+1} = -\sum_{i=1}^n\ell_i$. In this respect, zero-sum games are as complicated as general-sum games. In contrast, Hamiltonian games are much simpler than general-sum games. Theorem~\ref{t:conserve} shows that Hamiltonian games obey a conservation law -- which also provides the key to solving them, by gradient descent on the conserved quantity. 

\paragraph{Algorithms.}
The general case, neither potential nor Hamiltonian, is more difficult and is therefore the focus of the remainder of the paper. Section~\ref{s:algorithms} proposes \emph{symplectic gradient adjustment (SGA)}, a gradient-based method for finding stable fixed points in general games. Appendix~\ref{s:code} contains TensorFlow code to compute the adjustment. The algorithm computes two Jacobian-vector products, at a cost of two iterations of backprop. SGA satisfies a few natural desiderata explained in section~\ref{s:desiderata}: ($D1$) it is compatible with the original dynamics; and it is guaranteed to find stable equilibria in ($D2$) potential and ($D3$) Hamiltonian games.

For general games, correctly picking the \emph{sign} of the adjustment (whether to add or subtract) is critical since it determines the behavior near stable and unstable equilibria. Section~\ref{s:stab} defines stable equilibria and contrasts them with local Nash equilibria. Theorem \ref{t:convergence} proves that SGA converges locally to stable fixed points for sufficiently small parameters (which we quantify via the notion of an additive condition number). While strong, this may be impractical or slow down convergence significantly. Accordingly, lemma~\ref{l:sign_align} shows how to set the sign so as to be attracted towards stable equilibria and repelled from unstable ones. Correctly aligning SGA allows higher learning rates and faster, more robust convergence, see theorem~\ref{t:cosine}. Finally, theorem \ref{thm:avoid_saddle} tackles the remaining class of saddle fixed points by proving that SGA locally avoids strict saddles for appropriate parameters.

\paragraph{Experiments.}
We investigate the empirical performance of SGA in four basic experiments. The first experiment shows how increasing alignment allows higher learning rates and faster convergence, figure~\ref{f:learning_rates}. The second set of experiments compares SGA with optimistic mirror descent on two-player and four-player games. We find that SGA converges over a much wider range of learning rates.

The last two sets of experiments investigate mode collapse, mode hopping and the related, less well-known problem of boundary distortion identified in \citet{santurkar:18}. Mode collapse and mode hopping are investigated in a setup involving a two-dimensional mixture of 16 Gaussians that is somewhat more challenging than the original problem introduced in \citet{metz:17}. Whereas simultaneous gradient descent completely fails, our symplectic adjustment leads to rapid convergence -- slightly improved by correctly choosing the sign of the adjustment. 

Finally, boundary distortion is studied using a 75-dimensional spherical Gaussian. Mode collapse is not an issue since there the data distribution is unimodal. However, as shown in figure~\ref{f:hdg}, a vanilla GAN with RMSProp learns only one of the eigenvalues in the spectrum of the covariance matrix, whereas SGA approximately learns all of them.

The appendix provides some background information on differential and symplectic geometry, which motivated the developments in the paper. The appendix also explores what happens when the analogy with classical mechanics is pushed further than perhaps seems reasonable. We experiment with assigning units (in the sense of masses and velocities) to quantities in games, and find that type-consistency yields unexpected benefits. 

\subsection{Related work}

\citet{nash:50} was only concerned with existence of equilibria. Convergence in two-player games was studied in \citet{SinghKM00}. WoLF (Win or Learn Fast) converges to Nash equilibria in two-player two-action games \citep{bowling:02}. Extensions include weighted policy learning \citep{AbdallahL08} and GIGA-WoLF \citep{Bowling04}. Infinitesimal Gradient Ascent (IGA) is a gradient-based approach that is shown to converge to pure Nash equilibria in two-player two-action games. Cyclic behaviour may occur in case of mixed equilibria. \citet{Zinkevich} generalised the algorithm to $n$-action games called GIGA. Optimistic mirror descent approximately converges in two-player bilinear zero-sum games \citep{daskalakis:18}, a special case of Hamiltonian games. In more general settings it converges to coarse correlated equilibria.

Convergence has also been studied in various $n$-player settings, see \citet{rosen:65, scutari:10, facchinei:10, mertikopoulos:16}. However, the recent success of GANs, where the players are neural networks, has focused attention on a much larger class of \emph{nonconvex} games where comparatively little is known, especially in the $n$-player case. \citet{heusel:17} propose a two-time scale methods to find Nash equilibria. However, it likely scales badly with the number of players.  \citet{nagarajan:17} prove convergence for some algorithms, but under very strong assumptions \citep{mescheder:18}. Consensus optimization \citep{mescheder:17} is closely related to our proposad algorithm, and is extensively discussed in section~\ref{s:algorithms}. A variety of game-theoretically or minimax motivated modifications to vanilla gradient descent have been investigated in the context of GANs, see \citet{mertikopoulos:18a, gidel:18a}.

Learning with opponent-learning awareness (LOLA) infinitesimally modifies the objectives of players to take into account their opponents' goals \citep{foerster:18}. However, \citet{letcher:18} recently showed that LOLA modifies fixed points and thus fails to find stable equilibria in general games.

Symplectic gradient adjustment was independently discovered by \citet{gemp:18}, who refer to it as ``crossing-the-curl''. Their analysis draws on powerful techniques from variational inequalities and monotone optimization that are complementary to those developed here -- see for example \citet{gemp:16, gemp:17, gidel:18}.  Using techniques from monotone optimization, \citet{gemp:18} obtained more detailed and stronger results than ours, in the more particular case of Wasserstein LQ-GANs, where the generator is linear and the discriminator is quadratic \citep{feizi:17, nagarajan:17}.

Network zero-sum games are shown to be Hamiltonian systems in \citet{bailey:19}. The implications of the existence of invariant functions for games is just beginning to be understood and explored.

\paragraph{Notation.}
Dot products are written as $\vt^\intercal\wt$ or $\langle\vt,\wt\rangle$. The angle between two vectors is $\theta(\vt,\wt)$. Positive definiteness is denoted $\bS\succ0$.

\begin{figure}[t]  
    \center
    \includegraphics[width=.8\textwidth]{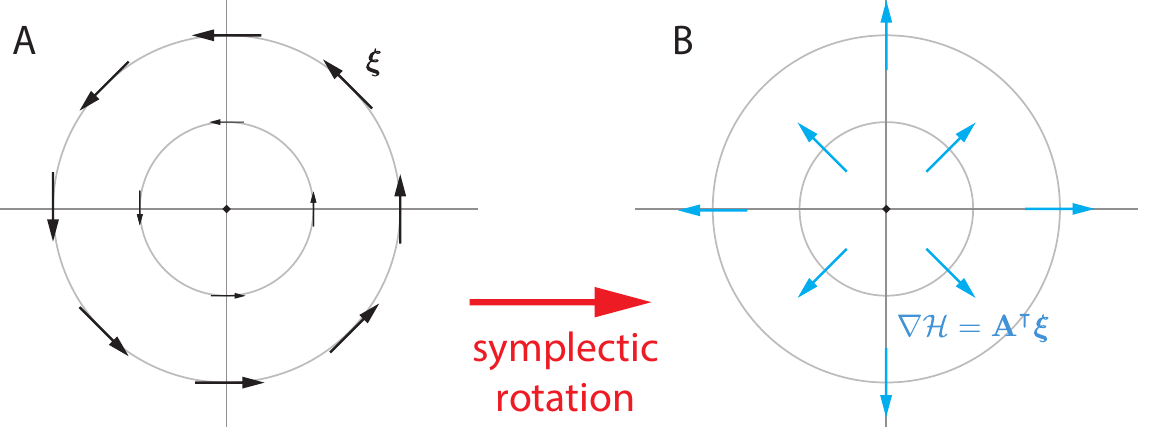}
    \caption{\emph{A minimal example of Hamiltonian mechanics.} 
    Consider a game where $\ell_1(x,y)=xy$, $\ell_2(x,y)=-xy$, and the dynamics are given by $\bxi(x,y)=(y,-x)$. The game is a special case of example~\ref{eg:basic_eg}.
    \textbf{(A)} The dynamics $\bxi$
    cycle around the origin since they live on the level sets of the Hamiltonian $\cH(x,y) = \frac{1}{2}(x^2+y^2)$. 
    \textbf{(B)} 
    Gradient descent on the Hamiltonian $\cH$ converges to the Nash equilibrium of the game, at the origin $(0,0)$. Note that $\bA^\intercal\bxi = (x,y)=\grad\cH$.
    }
    \label{f:hamiltonian}
\end{figure}

\section{The Infinitesimal Structure of Games}
\label{s:gt}

In contrast to the classical formulation of games, we do not constrain the parameter sets to the probability simplex or require losses to be convex in the corresponding players' parameters. Our motivation is that we are primarily interested in use cases where players are interacting neural nets such as GANs \citep{goodfellow:14}, a situation in which results from classical game theory do not straightforwardly apply.

\begin{defn}[differentiable game]\eod
  A differentiable game consists in a set of players $[n]=\{1,\ldots,n\}$ and corresponding twice continuously differentiable losses $\{\ell_i:\bR^d\rightarrow \bR\}_{i=1}^n$. Parameters are $\wt = (\wt_1,\ldots, \wt_n)\in\bR^d$ where $\sum_{i=1}^n d_i=d$. Player $i$ controls $\wt_i\in \bR^{d_i}$, and aims to minimize its loss.
\end{defn}
It is sometimes convenient to write $\wt=(\wt_i,\wt_{-i})$ where $\wt_{-i}$ concatenates the parameters of all the players other than the $i^\text{th}$, which is placed out of order by abuse of notation. 

The \emph{simultaneous gradient} is the gradient of the losses with respect to the parameters of the respective players:
\begin{equation}
  \label{e:dynamics}
  \bxi(\wt) = \left(\grad_{\wt_1}\ell_1, \ldots, \grad_{\wt_n}\ell_n\right) \in \bR^d.
\end{equation}
By the \textbf{dynamics} of the game, we mean following the \emph{negative} of the vector field, $-\bxi$, with infinitesimal steps. There is no reason to expect $\bxi$ to be the gradient of a \emph{single} function in general, and therefore no reason to expect the dynamics to converge to a fixed point.


\subsection{Hamiltonian Mechanics}
\label{s:ham_mech}

Hamiltonian mechanics is a formalism for describing the dynamics in classical physical systems, see \citet{arnold:89, guillemin:90}. The system is described via canonical coordinates $(\bq, \bp)$. For example, $\bq$ often refers to position and $\bp$ to momentum of a particle or particles. 

The Hamiltonian of the system $\cH(\bq,\bp)$ is a function that specifies the total energy as a function of the generalized coordinates. For example, in a closed system the Hamiltonian is given by the sum of the potential and kinetic energies of the particles. The time evolution of the system is given by Hamilton's equations:
\begin{equation}
    \frac{d\bq}{dt} = \frac{\dd \cH}{\dd\bp}
    \quad\text{and}\quad
    \frac{d\bp}{dt} = -\frac{\dd \cH}{\dd\bq}.
\end{equation}
An importance consequence of the Hamiltonian formalism is that the dynamics of the physical system -- that is, the trajectories followed by the particles in phase space -- live on the level sets of the Hamiltonian. In other words, the total energy is conserved.

\subsection{Hamiltonian Mechanics in Games}
\label{s:bim}

The next example illustrates the essential problem with gradients in games and the key insight motivating our approach.

\begin{eg}[Conservation of energy in a zero-sum unconstrained bimatrix game]\label{eg:basic_eg}
    Zero-sum games, where $\sum_{i=1}^n\ell_i\equiv 0$, are  well-studied. The zero-sum game
    \begin{equation}
        \ell_1(\x,\y) = \x^\intercal \bA\y
        \quad\text{and}\quad
        \ell_2(\x, \y) = -\x^\intercal \bA\y
    \end{equation}
    has a Nash equilibrium at $(\x,\y)=({\mathbf 0},{\mathbf 0})$. The simultaneous gradient $\bxi(\x,\y)=( \bA\y,-\bA^\intercal \x)$ rotates around the Nash, see figure~\ref{f:hamiltonian}. 
    
    The matrix $\bA$ admits singular value decomposition (SVD) $\bA= \bU^\intercal \bD \bV$. Changing to coordinates $\ut = \bD^{\frac{1}{2}}\bU\x$ and $\vt = \bD^{\frac{1}{2}}\bV\y$ gives $\ell_1(\ut,\vt) = \ut^\intercal\vt$ and $\ell_2(\ut,\vt) = -\ut^\intercal\vt$. Introduce the \emph{Hamiltonian}
    \begin{equation}
        \cH(\ut,\vt) = \frac{1}{2}\left(\|\ut\|_2^2 + \|\vt\|_2^2\right)
        = \frac{1}{2}\left(\x^\intercal \bU^\intercal \bD \bU\x
        + \y^\intercal \bV^\intercal \bD\bV\y\right).
    \end{equation}
    Remarkably, the dynamics can be reformulated via Hamilton's equations in the coordinates given by the SVD of $\bA$:
    \begin{equation}
        \bxi(\ut,\vt) 
        = \left(\frac{\dd \cH}{\dd \vt}, -\frac{\dd \cH}{\dd \ut}\right).
    \end{equation}
    The vector field $\bxi$ cycles around the equilibrium because $\bxi$ conserves the Hamiltonian's level sets (i.e. $\langle\bxi,\grad\cH\rangle=0$). However, \textbf{\emph{gradient descent on the Hamiltonian converges to the Nash equilibrium.}} The remainder of the paper explores the implications and limitations of this insight.
\end{eg}

\citet{papadimitriou:16} recently analyzed the dynamics of Matching Pennies (essentially, the above example) and showed that the cyclic behavior covers the entire parameter space. The Hamiltonian reformulation directly explains the cyclic behavior via a conservation law. 

\subsection{The Generalized Helmholtz Decomposition}

The \textbf{Jacobian} of a game with dynamics $\bxi$ is the $(d\times d)$-matrix of second-derivatives $\Jac(\wt) := \grad_\wt \cdot \bxi(\wt)^\intercal = \Big(\frac{\dd\xi_\alpha(\wt)}{\dd w_\beta}\Big)_{\alpha,\beta=1}^d$, where $\bxi_\alpha(\wt)$ is the $\alpha^\text{th}$ entry of the $d$-dimensional vector $\bxi(\wt)$. Concretely, the Jacobian can be written as
\begin{equation}
  \Jac(\wt) = \left(\begin{matrix}
    \grad^2_{\wt_1}\ell_1 & \grad^2_{\wt_1,\wt_2}\ell_1 & \cdots & \grad^2_{\wt_1,\wt_n}\ell_1 \\
    \grad^2_{\wt_2,\wt_1}\ell_2 & \grad^2_{\wt_2}\ell_2 & \cdots & \grad^2_{\wt_2,\wt_n}\ell_2 \\
    \vdots & & & \vdots \\
    \grad^2_{\wt_n,\wt_1}\ell_n & \grad^2_{\wt_n,\wt_2}\ell_n & \cdots & \grad^2_{\wt_n}\ell_n 
  \end{matrix}\right)
\end{equation}
where $\grad^2_{\wt_i,\wt_j}\ell_k$ is the $(d_i\times d_j)$-block of $2^\text{nd}$-order derivatives. The Jacobian of a game is a square matrix, but not necessarily symmetric. Note: Greek indices $\alpha,\beta$ run over $d$ parameter dimensions whereas Roman indices $i,j$ run over $n$ players.

\begin{lem}[generalized Helmholtz decomposition]\label{t:helmholtz}\eod
    The Jacobian of any vector field decomposes uniquely into two components $\Jac(\wt) = \bS(\wt) + \bA(\wt)$ where $\bS\equiv \bS^\intercal$ is symmetric and $\bA+\bA^\intercal\equiv 0$ is antisymmetric.
\end{lem}

\begin{proof}
    Any matrix decomposes uniquely as $\bM = \bS + \bA$ where $\bS = \frac{1}{2} (\bM + \bM^\intercal)$ and $\bA = \frac{1}{2} (\bM - \bM^\intercal)$ are symmetric and antisymmetric. The decomposition is preserved by orthogonal change-of-coordinates: given orthogonal matrix $\bP$, we have $\bP^\intercal \bM \bP = \bP^\intercal \bS \bP + \bP^\intercal \bA \bP$ since the terms remain symmetric and antisymmetric. Applying the decomposition to the Jacobian yields the result.
\end{proof}


The connection to the classical Helmholtz decomposition in calculus is sketched in appendix~\ref{s:diff}. Two natural classes of games arise from the decomposition: 

\begin{defn}
    A game is a \textbf{potential game} if the Jacobian is symmetric, i.e. if $\bA(\wt)\equiv0$. It is a \textbf{Hamiltonian game} if the Jacobian is antisymmetric, i.e. if $\bS(\wt)\equiv0$.
\end{defn}
Potential games are well-studied and easy to solve. Hamiltonian games are a new class of games that are also easy to solve. The general case is more difficult, see section~\ref{s:algorithms}.

\subsection{Stable Fixed Points (SFPs) vs Local Nash Equilibria (LNEs)}
\label{s:stab}

There are (at least) two possible solution concepts in general differentiable games: stable fixed points and local Nash equilibria.

\begin{defn}
    A point $\wt^∗$ is a \textbf{local Nash equilibrium} if, for all $i$, there exists a neighborhood $U_i$ of $\wt_i^∗$ such that $\ell_i(\wt_i', \wt_{-i}^∗) \geq \ell_i(\wt_i^∗, \wt_{-i}^∗)$ for $\wt'_i \in U_i$.
\end{defn}

We introduce \emph{local} Nash equilibria because finding \emph{global} Nash equilibria is unrealistic in games involving neural nets. Gradient-based methods can reliably find local -- but not global -- optima of nonconvex objective functions \citep{lee:16,lee:17}. Similarly, gradient-based methods cannot be expected to find global Nash equilibria in nonconvex games.


\begin{defn}
    A fixed point $\wt^*$ with $\bxi(\wt^*)=0$ is \textbf{\emph{stable}} if $\Jac(\wt^*)\succeq 0$ and $\Jac(\wt^*)$ is invertible, \textbf{\emph{unstable}} if $\Jac(\wt^*)\prec0$ and a \textbf{\emph{strict saddle}} if $\Jac(\wt^*)$ has an eigenvalue with negative real part. Strict saddles are a subset of unstable fixed points.
\end{defn}

The definition is adapted from \citet{letcher:18}, where conditions on the Jacobian hold \emph{at} the fixed point; in contrast, \citet{dgm:18} imposed conditions on the Jacobian in a \emph{neighborhood} of the fixed point. We motivate this concept as follows.

Positive semidefiniteness, $\Jac(\wt^*) \succeq 0$, is a minimal condition for any reasonable notion of stable fixed point. In the case of a single loss $\ell$, the Jacobian of $\bxi = \grad\ell$ is the Hessian of $\ell$, i.e. $\Jac = \grad^2\ell$. Local convergence of gradient descent on single functions cannot be guaranteed if $\Jac(\wt^*) \nsucceq 0$, since such points are strict saddles. These are almost always avoided by \citet{lee:17}, so this semidefinite condition must hold.

Another viewpoint is that invertibility and positive semidefiniteness of the Hessian together imply \textit{positive definiteness}, and the notion of stable fixed point specializes, in a one-player game, to local minima that are detected by the second partial derivative test. These minima are precisely those which gradient-like methods provably converge to. Stable fixed points are defined by analogy, though note that invertibility and semidefiniteness do \textit{not} imply positive definiteness in $n$-player games since $\Jac$ may not be symmetric.

Finally, it is important to impose only positive \emph{semi-}definiteness to keep the class as large as possible. Imposing strict positivity would imply that the origin is not an SFP in the cyclic game $\ell_1 = xy = -\ell_2$ from Example \ref{eg:basic_eg}, while clearly deserving of being so.

\begin{rem}\label{rem:stable}
    The conditions $\Jac(\wt^*)\succeq 0$ and $\Jac(\wt^*)\prec 0$ are equivalent to the conditions on the symmetric component $\bS(\wt^*)\succeq 0$ and $\bS(\wt^*) \prec 0$ respectively, since
    \begin{equation}
        \ut^\intercal \Jac \ut = \ut^\intercal \bS \ut + \ut^\intercal \bA \ut = \ut^\intercal \bS \ut
    \end{equation}
    for all $\ut$, by antisymmetry of $\bA$. This equivalence will be used throughout.
\end{rem}

Stable fixed points and local Nash equilibria are both appealing solution concepts, one from the viewpoint of optimisation by analogy with single objectives, and the other from game theory. Unfortunately, neither is a subset of the other:

\begin{eg}[stable $\centernot\implies$ local Nash]\label{eg:stable_notnash}\eod
    Let $\ell_1(x, y) = x^3+xy$ and $\ell_2(x,y) = -xy$. Then
  \begin{equation}
    \bxi(x,y) = \left(\begin{matrix}
      3x^2+y \\ -x
    \end{matrix}\right)
    \quad\text{and}\quad
    \Jac(x,y) = \left(\begin{matrix}
      6x & 1 \\ -1 & 0
    \end{matrix}\right).
  \end{equation}
  There is a stable fixed point with invertible Hessian at $(x,y) = (0,0)$, since $\bxi(0,0) = 0$ and $\Jac(0,0) \succeq 0$ invertible. However any neighbourhood of $x=0$ contains some small $\epsilon > 0$ for which $\ell_1(-\epsilon, 0) = -\epsilon^3 < 0 = \ell_1(0,0)$, so the origin is not a local Nash equilibrium.
\end{eg}

\begin{eg}[local Nash $\centernot\implies$ stable]\label{eg:nash_notstable}\eod
  Let $\ell_1(x, y) = \ell_2(x, y) = xy$. Then
  \begin{equation}
    \bxi(x,y) = \left(\begin{matrix}
      y \\ x
    \end{matrix}\right)
    \quad\text{and}\quad
    \Jac(x,y) = \left(\begin{matrix}
      0 & 1 \\ 1 & 0
    \end{matrix}\right).
  \end{equation}
  There is a fixed point at $(x,y) = (0,0)$ which is a local (in fact, global) Nash equilibrium since $\ell_1(0, y) = 0 \geq \ell_1(0, 0)$ and $\ell_2(x, 0) = 0 \geq \ell_2(0, 0)$ for all $x,y \in \R$. However $\Jac = \bS$ has eigenvalues $\lambda_1 = 1$ and $\lambda_2 = -1 < 0$, so $(0,0)$ is not a stable fixed point.
\end{eg}

In Example \ref{eg:nash_notstable}, the Nash equilibrium is a \emph{saddle point} of the common loss $\ell = xy$. Any algorithm that converges to Nash equilibria will thus converge to an undesirable saddle point. This rules out local Nash equilibrium as a solution concept for our purposes. Conversely, Example 2 emphasises the better notion of stability whereby player 1 may have a local incentive to deviate from the origin \textit{immediately}, but would later be punished for doing so since the game is locally dominated by the $\pm xy$ terms, whose only `resolution' or `stable minimum' is the origin (see Example \ref{eg:basic_eg}).

\subsection{Potential Games}

Potential games were introduced by \citet{monderer:96}. It turns out that our definition of potential game above coincides with a special case of the potential games of \citet{monderer:96}, which they refer to as exact potential games.

\begin{defn}[classical definition of potential game]\eod
    A game is a potential game if there is a single potential function $\phi:\bR^d\rightarrow\bR$ and positive numbers $\{\alpha_i>0\}_{i=1}^n$ such that 
    \begin{equation}
      \phi(\wt_i', \wt_{-i})  - \phi(\wt_i'',\wt_{-i})
      = \alpha_i\Big(\ell_i(\wt_i',\wt_{-i}) - \ell_i(\wt_i'',\wt_{-i})\Big)
    \end{equation}
    for all $i$ and all $\wt_i', \wt_i'', \wt_{-i}$, see \citet{monderer:96}.
\end{defn}

\begin{lem}
    A game is a potential game iff $\alpha_i\grad_{\wt_i} \ell_i = \grad_{\wt_i} \phi$ for all $i$, which is equivalent to
    \begin{equation}\label{eq:pot_sym}
      \alpha_i\grad^2_{\wt_i\wt_j} \ell_i
      = \alpha_j\grad^2_{\wt_i \wt_j}\ell_j 
      = \alpha_j\left(\grad^2_{\wt_j \wt_i}\ell_j \right)^\intercal
      \quad\forall i,j.
    \end{equation}
\end{lem}

\begin{proof}
    See \citet{monderer:96}.
\end{proof}
 
\begin{cor}\label{cor:potpot}
    If $\alpha_i=1$ for all $i$ then equation~\eqref{eq:pot_sym} is equivalent to requiring that the Jacobian of the game is symmetric. 
\end{cor}

\begin{proof}
    In an exact potential game, the Jacobian coincides with the Hessian of the potential function $\phi$, which is necessarily symmetric.
\end{proof}

\citet{monderer:96} refer to the special case where $\alpha_i=1$ for all $i$ as an \emph{\textbf{exact potential game}}. We use the shorthand `potential game' to refer to exact potential games in what follows.

Potential games have been extensively studied since they are one of the few classes of games for which Nash equilibria can be computed \citep{rosenthal:73}. For our purposes, they are games where simultaneous gradient descent on the losses corresponds to gradient descent on a single function. It follows that descent on $\bxi$ converges to a fixed point that is a local minimum of $\phi$ or a saddle.

\subsection{Hamiltonian Games}

Hamiltonian games, where the Jacobian is antisymmetric, are a new class games. They are  related to the harmonic games introduced in \citet{candogan:11}, see section~\ref{s:harmonic}. An example from \citet{reval:18} may help develop intuition for antisymmetric matrices: 

\begin{eg}[antisymmetric structure of tournaments]\eod
    Suppose $n$ competitors play one-on-one and that the probability of player $i$ beating player $j$ is $p_{ij}$. Then, assuming there are no draws, the probabilities satisfy $p_{ij}+p_{ji}=1$ and $p_{ii}=\frac{1}{2}$. The matrix $\bA=\left(\log\frac{p_{ij}}{1-p_{ij}}\right)_{i,j=1}^n$ of logits is then antisymmetric. Intuitively, antisymmetry reflects a \emph{hyperadversarial} setting where all pairwise interactions between players are zero-sum. 
\end{eg}

Hamiltonian games are closely related to zero-sum games.

\begin{eg}[an unconstrained bimatrix game is zero-sum iff it is Hamiltonian]
    Consider bimatrix game with $\ell_1(\x,\y) = \x^\intercal \bP\y$ and $\ell_2(\x,\y)=\x^\intercal \bQ\y$, but where the parameters are \emph{not} constrained to the probability simplex.
   Then $\bxi = (\bP\y, \bQ^\intercal \x)$ and the Jacobian components have block structure
    \begin{equation}
        \bA = \frac{1}{2}\left(\begin{matrix}
            0 & \bP - \bQ \\
            (\bQ - \bP)^\intercal & 0
        \end{matrix}\right)
        \quad\text{and}\quad
        \bS=\frac{1}{2}\left(\begin{matrix}
            0 & \bP+\bQ \\
            (\bP+\bQ)^\intercal & 0
        \end{matrix}\right)
    \end{equation}
    The game is Hamiltonian iff $\bS=0$ iff $\bP + \bQ=0$ iff $\ell_1 + \ell_2=0$.
\end{eg}
 
 However, in general there are Hamiltonian games that are \emph{not} zero-sum and vice versa.

\begin{eg}[Hamiltonian game that is not zero-sum]\label{eg:ham}\eod
    Fix constants $a$ and $b$ and suppose players $1$ and $2$ minimize losses
  \begin{equation}
    \ell_1(x,y) = x(y-b)
    \,\text{ and }\,
    \ell_2(x,y) = -(x-a)y
  \end{equation}
  with respect to $x$ and $y$ respectively.
\end{eg}

\begin{eg}[zero-sum game that is not Hamiltonian]\label{eg:not_hamiltonian}\eod
    Players 1 and 2 minimize
    \begin{equation}
        \ell_1(x,y) = x^2 + y^2
        \quad\quad
        \ell_2(x,y) = -(x^2 +y^2).
    \end{equation}
    The game actually has potential function $\phi(x,y) = x^2 - y^2$.
\end{eg}
Hamiltonian games are quite different from potential games. In a Hamiltonian game there is a Hamiltonian function $\cH$ that specifies a conserved quantity. In potential games the dynamics \emph{equal} $\grad\phi$; in Hamiltonian games the dynamics are \emph{orthogonal} to $\grad\cH$. The orthogonality implies the conservation law that underlies the cyclic behavior in example~\ref{eg:basic_eg}.

\begin{thm}[conservation law for Hamiltonian games]\label{t:conserve}\eod
    Let $\cH(\wt) := \frac{1}{2}\|\bxi(\wt)\|^2_2$. If the game is Hamiltonian then
    \begin{enumerate}[i)]
        \item $\grad \cH = \bA^\intercal \bxi$ and
        \item \textbf{$\bxi$ preserves the level sets of $\cH$} since $\langle\bxi, \grad\cH\rangle =0$.
        \item If the Jacobian is invertible and $\lim_{\|\wt\|\rightarrow\infty}\cH(\wt)=\infty$ then gradient descent on $\cH$ converges to a stable fixed point.
    \end{enumerate}
\end{thm}

\begin{proof}
    Direct computation shows $\grad \cH = \Jac^\intercal \bxi$ for any game. The first statement follows since $\Jac=\bA$ in Hamiltonian games. 
    
    For the second statement, the directional derivative is $D_{\bxi} \cH= \langle\bxi,\grad \cH\rangle = \bxi^\intercal \bA^\intercal \bxi$ where $\bxi^\intercal \bA^\intercal \bxi = (\bxi^\intercal \bA^\intercal \bxi)^\intercal = \bxi^\intercal \bA\bxi = -(\bxi^\intercal \bA^\intercal\bxi)$ since $\bA=-\bA^\intercal$ by anti-symmetry. It follows that $\bxi^\intercal \bA^\intercal\bxi = 0$.
    
    For the third statement, gradient descent on $\cH$ will converge to a point where $\grad\cH = \Jac^\intercal\bxi(\wt) = 0$. If the Jacobian is invertible then clearly $\bxi(\wt)=0$. The fixed-point is stable since $0\equiv\bS\succeq0$ in a Hamiltonian game, recall remark~\ref{rem:stable}.
\end{proof}

In fact, $\cH$ is a Hamiltonian function for the game dynamics, see appendix~\ref{s:diff} for a concise explanation. We use the notation $\cH(\wt)= \frac{1}{2} \|\bxi(\wt)\|^2$ throughout the paper. However, $\cH$ can only be interpreted as a Hamiltonian function for $\bxi$ when the game is Hamiltonian.

There is a precise mapping from Hamiltonian games to symplectic geometry, see appendix~\ref{s:diff}. Symplectic geometry is the modern formulation of classical mechanics \citep{arnold:89, guillemin:90}. Recall that periodic behaviors (e.g. orbits) often arise in classical mechanics. The orbits lie on the level sets of the Hamiltonian, which expresses the total energy of the system.  

\section{Algorithms}
\label{s:algorithms}

We have seen that fixed points of potential and Hamiltonian games can be found by descent on $\bxi$ and $\grad\cH$ respectively. This section tackles finding stable fixed points in general games.

\subsection{Finding Stable Fixed Points} 
\label{s:desiderata}
There are two classes of games where we know how to find stable fixed points: potential games where $\bxi$ converges to a local minimum and Hamiltonian games where $\grad\cH$, which is orthogonal to $\bxi$, finds stable fixed points. 

In the general case, the following desiderata provide a set of reasonable properties for an adjustment $\bxi_\lambda$ of the game dynamics. Recall that $\theta(\ut,\vt)$ is the angle between the vectors $\ut$ and $\vt$. 

\paragraph{Desiderata.}
To find stable fixed points, an adjustment $\bxi_\lambda$ to the game dynamics should satisfy
    \begin{enumerate}[D1.]
      \item \emph{compatible\footnote{Two nonzero vectors are compatible if they have positive inner product. } with game dynamics:}
      $\langle \bxi_\lambda, \bxi\rangle =\alpha_1\cdot  \|\bxi\|^2$;
      
      \item \emph{compatible with potential dynamics:}\\
      if the game is a potential game then $\langle\bxi_\lambda,\grad\phi\rangle=\alpha_2\cdot \|\grad\phi\|^2$;
      
      \item \emph{compatible with Hamiltonian dynamics:}\\
      If the game is Hamiltonian then $\langle \bxi_\lambda,\grad\cH\rangle = \alpha_3\cdot \|\grad\cH\|^2$;
      
      \item \emph{attracted to stable equilibria:}\\
      in neighborhoods where $\bS\succ 0$, require $\theta(\bxi_\lambda,\grad\cH)\leq \theta(\bxi,\grad\cH)$;
        
        \item \emph{repelled by unstable equilibria:}\\
      in neighborhoods where $\bS\prec 0$, require  $\theta(\bxi_\lambda,\grad\cH)\geq \theta(\bxi,\grad\cH)$.
    \end{enumerate}
    for some $\alpha_1,\alpha_2,\alpha_3>0$.
    
\vspace{3mm}
Desideratum $D1$ does not guarantee that players act in their own self-interest -- this requires a stronger positivity condition on dot-products with subvectors of $\bxi$, see \citet{stg:17}. Desiderata $D2$ and $D3$ imply that the adjustment behaves correctly in potential and Hamiltonian games respectively.

To understand desiderata $D4$ and $D5$, observe that gradient descent on $\cH = \frac{1}{2} \|\bxi\|^2$ will find local minima that are fixed points of the dynamics. However, we specifically wish to converge to stable fixed points. Desideratum $D4$ and $D5$ require that the adjustment improves the rate of convergence to stable fixed points (by finding a steeper angle of descent), and avoids unstable fixed points. 

More concretely, desiderata $D4$ can be interpreted as follows. If $\bxi$ points at a stable equilibrium then we require that $\bxi_\lambda$ points \emph{more} towards the equilibrium (i.e. has smaller angle). Conversely, desiderata $D5$ requires that if $\bxi$ points away then the adjustment should point \emph{further} away. 

The unadjusted dynamics $\bxi$ satisfies all the desiderata except $D3$. 

\subsection{Consensus Optimization}

Since gradient descent on the function $\cH(\wt) = \frac{1}{2}\|\bxi\|^2$ finds stable fixed points in Hamiltonian games, it is natural to ask how it performs in general games. If the Jacobian $\Jac(\wt)$ is invertible, then $\grad\cH = \Jac^\intercal\bxi = 0$ iff $\bxi=0$. Thus, gradient descent on $\cH$ converges to fixed points of $\bxi$.

However, there is no guarantee that descent on $\cH$ will find a \emph{stable} fixed point. \citet{mescheder:17} propose \emph{consensus optimization}, a gradient adjustment of the form 
\begin{equation}
    \bxi + \lambda\cdot \Jac^\intercal\bxi = \bxi + \lambda\cdot \grad\cH.
\end{equation}
Unfortunately, consensus optimization can converge to unstable fixed points even in simple cases where the `game' is to minimize a single function:
\begin{eg}[consensus optimization can converge to a global maximum]\label{eg:con_fail}\eod
  Consider a potential game with losses $\ell_1(x,y)=\ell_2(x,y)= -\frac{\kappa}{2}(x^2 + y^2)$  with $\kappa\gg 0$. Then
  \begin{equation}
    \bxi = -\kappa\cdot \left(\begin{matrix}
      x \\ y
    \end{matrix}\right)
    \,\text{ and }\,
    \Jac = -\left(\begin{matrix}
      \kappa & 0 \\
      0 & \kappa
    \end{matrix}\right)
  \end{equation}  
  Note that $\|\bxi\|^2 = \kappa^2 (x^2 + y^2)$ and 
  \begin{equation}
    \bxi + \lambda\cdot \Jac^\intercal \bxi 
    = \kappa( \lambda\kappa - 1) \cdot \left(\begin{matrix}
      x \\
      y
    \end{matrix}\right).
  \end{equation}
  Descent on $\bxi + \lambda\cdot \Jac^\intercal \bxi$ converges to the global maximum $(x,y)=(0,0)$ unless $\lambda<\frac{1}{\kappa}$.
\end{eg}
Although consensus optimization works well in two-player zero-sum, it cannot be considered a candidate algorithm for finding stable fixed points in general games since it fails in the basic case of potential games. Consensus optimization only satisfies desiderata $D3$ and $D4$.

\subsection{Symplectic Gradient Adjustment}

The problem with consensus optimization is that it can perform worse than gradient descent on potential games. Intuitively, it makes bad use of the symmetric component of the Jacobian. Motivated by the analysis in section~\ref{s:gt}, we propose symplectic gradient adjustment, which takes care to only use the antisymmetric component of the Jacobian when adjusting the dynamics.

\begin{prop}\label{p:sym_desc}
    The \textbf{symplectic gradient adjustment (SGA)}
    \begin{equation}
        \bxi_\lambda := \bxi + \lambda \cdot \bA^\intercal \bxi.
    \end{equation}
    satisfies $D1$--$D3$ for $\lambda>0$, with $\alpha_1=1=\alpha_2$ and $\alpha_3=\lambda$.
\end{prop}
\begin{proof}
    First claim: $\lambda\cdot \bxi^\intercal \bA^\intercal\bxi = 0$ by anti-symmetry of $\bA$. Second claim:  $\bA\equiv 0$ in a potential game, so $\bxi_\lambda=\bxi=\grad\phi$. Third claim: $\langle\bxi_\lambda,\grad\cH\rangle=\langle\bxi_\lambda, \Jac^\intercal\bxi\rangle=\langle\bxi_\lambda, \bA^\intercal\bxi\rangle = \lambda\cdot \bxi^\intercal\bA\bA^\intercal\bxi =\lambda\cdot \|\grad\cH\|^2$ since $\Jac=\bA$ by assumption.
\end{proof}
Note that desiderata $D1$ and $D2$ are true even when $\lambda<0$. This will prove useful, since example~\ref{eg:get_sign_right} shows that it may be necessary to pick negative $\lambda$ near $\bS\prec 0$. Section~\ref{s:pick_l} shows how to also satisfy desiderata $D4$ and $D5$.

\subsection{Convergence}

We begin by analysing convergence of SGA near stable equilibria. The following lemma highlights that the interaction between the symmetric and antisymmetric components is important for convergence. Recall that two matrices $\bA$ and $\bS$ \emph{commute} iff $[\bA,\bS] := \bA\bS - \bS\bA = {\mathbf 0}$. That is, $\bA$ and $\bS$ commute iff $\bA\bS = \bS\bA$. Intuitively, two matrices commute if they have the same preferred coordinate system. 

\begin{lem}\label{l:noncomm}
  If $\bS\succeq 0$ is symmetric positive semidefinite and $\bS$ commutes with $\bA$ then $\bxi_\lambda$ points towards stable fixed points for non-negative $\lambda$:
  \begin{equation}
    \langle\bxi_\lambda,\grad \cH\rangle \geq 0
    \text{ for all $\lambda\geq 0$}.
  \end{equation}
\end{lem}


\begin{proof}
    First observe that $\bxi^\intercal \bA\bS\bxi=\bxi^\intercal \bS^\intercal \bA^\intercal\bxi = -\bxi^\intercal \bS \bA\bxi$, where the first equality holds since the expression is a scalar, and the second holds since $\bS=\bS^\intercal$ and $\bA=-\bA^\intercal$. It follows that $\bxi^\intercal \bA\bS\bxi=0$ if $\bS\bA=\bA\bS$. Finally rewrite the inequality as
  \begin{align}
    \langle\bxi_\lambda,\grad \cH\rangle
    = \langle\bxi + \lambda\cdot \bA^\intercal \bxi, \bS\bxi + \bA^\intercal\bxi\rangle 
    = \bxi^\intercal \bS\bxi + \lambda \bxi^\intercal \bA\bA^\intercal\bxi
    \geq 0
  \end{align}
  since $\bxi^\intercal \bA \bS\bxi = 0$ and by positivity of $\bS$, $\lambda$ and $\bA\bA^\intercal$.
\end{proof}

The lemma suggests that in general the \emph{failure} of $\bA$ and $\bS$ to commute should be important for understanding the dynamics of $\bxi_\lambda$. We therefore introduce the \textbf{additive condition number} $\kappa$ to upper-bound the worst-case noncommutativity of $\bS$, which allows to quantify the relationship between $\bxi_\lambda$ and $\grad\cH$. If $\kappa=0$, then $\bS=\sigma\cdot {\mathbf I}$ commutes with \emph{all} matrices. The larger the additive condition number $\kappa$, the larger the \emph{potential} failure of $\bS$ to commute with other matrices. 


\begin{thm}\label{t:scalar_lambda}
    Let $\bS$ be a symmetric matrix with eigenvalues $\sigma_\text{max}\geq\cdots\geq \sigma_\text{min}$. The \textbf{additive condition number}\footnote{The condition number of a positive definite matrix is $\frac{\sigma_\text{max}}{\sigma_\text{min}}$.} of $\bS$ is $\kappa := \sigma_\text{max}-\sigma_\text{min}$.
  If $\bS\succeq 0$ is positive semidefinite with additive condition number $\kappa$ then $\lambda\in(0,\frac{4}{\kappa})$ implies
  \begin{equation}
    \label{eq:key}
    \langle\bxi_\lambda,\grad \cH\rangle \geq 0.
  \end{equation}
  If $\bS$ is negative semidefinite, then $\lambda\in(0,\frac{4}{\kappa})$ implies
  \begin{equation}
    \label{eq:key-neg}
    \langle\bxi_{-\lambda},\grad \cH\rangle \leq 0.
  \end{equation}
  The inequalities are strict if $\Jac$ is invertible.
\end{thm}

\begin{proof}
    We prove the case $\bS\succeq 0$; the case $\bS\preceq 0$ is similar.
  Rewrite the inequality as
  \begin{align}
    \langle\bxi+\lambda\cdot \bA^\intercal\bxi,\grad \cH\rangle
    & = (\bxi + \lambda\cdot \bA^\intercal \bxi)^\intercal \cdot (\bS + \bA^\intercal)\bxi \\
    & = \bxi^\intercal \bS\bxi + \lambda \bxi^\intercal \bA \bS \bxi + \lambda \bxi^\intercal \bA\bA^\intercal\bxi
  \end{align}
  Let $\beta = \|A^\intercal\bxi\|$ and $\tilde{\bS} = \bS - \sigma_\text{min}\cdot {\mathbf I}$, where ${\mathbf I}$ is the identity matrix. Then 
  \begin{equation}
    \bxi^\intercal \bS\bxi + \lambda \bxi^\intercal \bA \bS \bxi + \lambda \cdot\beta
    ^2 
    \geq \bxi^\intercal \tilde{\bS}\bxi + \lambda \bxi^\intercal \bA \tilde{\bS} \bxi + \lambda \cdot\beta^2
  \end{equation}
  since $\bxi^\intercal \bS\bxi\geq \bxi^\intercal \tilde{\bS}\bxi$ by construction and $\bxi^\intercal \bA \tilde{\bS} \bxi = \bxi^\intercal \bA \bS \bxi - \sigma_\text{min} \bxi^\intercal \bA\bxi = \bxi^\intercal \bA \bS \bxi$ because $\bxi^\intercal \bA\bxi = 0$ by the anti-symmetry of $\bA$. It therefore suffices to show that the inequality holds when $\sigma_\text{min}=0$ and $\kappa = \sigma_\text{max}$.
    
    Since $\bS$ is positive semidefinite, there exists an upper-triangular square-root matrix $T$ such that $\bT^\intercal \bT = \bS$ and so $\bxi^\intercal \bS\bxi = \|\bT\bxi\|^2$.
    Further, 
  \begin{equation}
    |\bxi^\intercal \bA\bS\bxi| \leq \|\bA^\intercal\bxi\|\cdot \|\bT^\intercal \bT\bxi\| \leq \sqrt{\sigma_\text{max}}\cdot \|\bA^\intercal\bxi\|\cdot \|\bT\bxi\|.
  \end{equation}
  since $\|\bT\|_2=\sqrt{\sigma_\text{max}}$. Putting the observations together obtains
  \begin{equation}
      \begin{split}
          \|\bT\bxi\|^2 + \lambda(\|\bA\bxi\|^2 - \langle \bA\bxi, \bS\bxi\rangle) 
          & \geq  \|\bT\bxi\|^2 + \lambda(\|\bA\bxi\|^2 - \|\bA\bxi\|~\|\bS\bxi\| \\
          & \geq  \|\bT\bxi\|^2 + \lambda\|\bA\bxi\|(\|\bA\bxi\| - \|\bS\bxi\|) \\
          & \geq  \|\bT\bxi\|^2 + \lambda\|\bA\bxi\|(\|\bA\bxi\| - \sqrt{\sigma_{max}}\|\bT\bxi\|)
      \end{split}
  \end{equation}
  Set $\alpha = \sqrt{\lambda}$ and $\eta = \sqrt{\sigma_{max}}$. We can continue the above computation
  \begin{equation}
      \begin{split}
          \|\bT\bxi\|^2 + \lambda(\|\bA\bxi\|^2 - \langle \bA\bxi, \bS\bxi\rangle) & \geq  \|\bT\bxi\|^2 + \alpha^2\|\bA\bxi\|(\|\bA\bxi\| - \eta\|\bT\bxi\|) \\
          & =  \|\bT\bxi\|^2 + \alpha^2\|\bA\bxi\|^2 - \alpha^2\|\bA\bxi\|\eta\|\bT\bxi\| \\
          & =  (\|\bT\bxi\| - \alpha\|\bA\bxi\|)^2 + 2\alpha\|\bA\bxi\|~\|\bT\bxi\| - \alpha^2\eta\|A\bxi\|~\|\bT\bxi\| \\
          & =  (\|\bT\bxi\| - \alpha\|\bA\bxi\|)^2 + \|\bA\bxi\|~\|\bT\bxi\|(2\alpha - \alpha^2\eta)
      \end{split}
  \end{equation}
  Finally,  $2\alpha - \alpha^2\eta > 0$ for any $\alpha$ in
  the range $(0, \frac{2}{\eta})$, which is to say, for any $0 < \lambda < \frac{4}{\sigma_{max}}$. The kernel of $\bS$ and
  the kernel of $\bT$ coincide. If $\bxi$ is in the kernel of $\bA$, resp. $\bT$, it cannot be in the kernel of $\bT$, resp. $\bA$ and
  the term $(\|\bT\bxi\| - \alpha\|\bA\bxi\|)^2$ is positive. Otherwise, the term $\|\bA\bxi\|\|\bT\bxi\|$ is positive.
\end{proof}

The theorem above guarantees that SGA always points in the direction of stable fixed points for $\lambda$ sufficiently small. This does not technically guarantee convergence; we use Ostrowski's theorem to strengthen this formally. Applying Ostrowski's theorem will require taking a more abstract perspective by encoding the adjusted dynamics into a differentiable map $F : \Omega \rightarrow \R^d$ of the form $F(\wt) = \wt-\alpha\bxi_\lambda(\wt)$.

\begin{thm}[Ostrowski]\label{th_c1}
    Let $F : \Omega \rightarrow \R^d$ be a continuously differentiable map on an open subset $\Omega \subseteq \R^d$, and assume $\wt^* \in \Omega$ is a fixed point. If all eigenvalues of $\nabla F(\wt^*)$ are strictly in the unit circle of $\mathbb{C}$, then there is an open neighbourhood $U$ of $\wt^*$ such that for all $\wt_0 \in U$, the sequence $F^{k}(\wt_0)$ of iterates of $F$ converges to $\wt^*$. Moreover, the rate of convergence is at least linear in $k$.
\end{thm}
\begin{proof}
    This is a standard result on fixed-point iterations, adapted from \citet[10.1.3]{Ort}.
\end{proof}

\begin{cor}\label{prop2.0}
    A matrix $\bM$ is called \textit{positive stable} if all its eigenvalues have positive real part. Assume $\wt^*$ is a fixed point of a differentiable game such that $(\indicator + \la \bA^\intercal)\Jac(\wt^*)$ is positive stable for $\la$ in some set $\Lambda$. Then SGA converges locally to $\wt^*$ for $\la \in \Lambda$ and $\al > 0$ sufficiently small.
\end{cor}

\begin{proof}
    Let $X = (\indicator + \la \bA^\intercal)$. By definition of fixed points, $\bxi(\wt^*) = 0$ and so
    \[ \nabla [X\bxi](\wt^*) = \nabla X(\wt^*) \bxi(\wt^*) + X(\wt^*)\nabla\bxi(\wt^*) = X\Jac(\wt^*) \]
    is positive stable by assumption, namely has eigenvalues $a_k+ib_k$ with $a_k > 0$. Writing $F(\wt) = \wt - \al X\bxi(\wt)$ for the iterative procedure given by SGA, it follows that
    \[ \nabla F(\wt^*) = \indicator - \al \nabla [X\bxi](\wt^*) \]
    has eigenvalues $1-\al a_k - i\al b_k$, which are in the unit circle for small $\al$. More precisely,
    \begin{align*}
        \left| 1-\al a_k - i\al b_k \right|^2 < 1 
        & \quad \iff \quad  \, 1-2\al a_k + \al^2 a_k^2 + \al^2 b_k^2 < 1 \\
        & \quad \iff \quad  \, 0 < \al < \frac{2a_k}{a_k^2 + b_k^2}
    \end{align*}
    which is always possible for $a_k > 0$. Hence $\nabla F(\wt^*)$ has eigenvalues in the unit circle for $0 < \al < \min_k 2a_k/(a_k^2 + b_k^2)$, and we are done by Ostrowski's Theorem since $\wt^*$ is a fixed point of $F$.
\end{proof}

\begin{thm}\label{t:convergence}
    Let $\wt^*$ be a stable fixed point and $\kappa$ the additive condition number of $\bS(\wt^*)$. Then SGA converges locally to $\wt^*$ for all $\lambda\in(0,\frac{4}{\kappa})$ and $\alpha > 0$ sufficiently small.
\end{thm}

\begin{proof}
    By Theorem 5 and the assumption that $\wt^*$ is a stable fixed point with invertible Jacobian, we know that
    \begin{equation}
      \langle\bxi_\lambda,\grad \cH\rangle = \langle (\indicator + \lambda \bA^\intercal)\bxi, \Jac^\intercal \bxi \rangle > 0
    \end{equation}
    for $\lambda\in(0,\frac{4}{\kappa})$. The proof does not rely on any particular property of $\bxi$, and can trivially be extended to the claim that
    \begin{equation}
        \langle (\indicator + \lambda \bA^\intercal)\ut, \Jac^\intercal \ut \rangle > 0
    \end{equation}
    for all non-zero vectors $\ut$. In particular this can be rewritten as
    \begin{equation}
        \ut^\intercal \Jac(\indicator + \lambda \bA^\intercal) \ut \rangle > 0 \,,
    \end{equation}
    which implies positive definiteness of $\Jac(\indicator + \lambda \bA^\intercal)$. A positive definite matrix is positive stable, and any matrices $AB$ and $BA$ have identical spectrum. This implies also that $(\indicator + \lambda \bA^\intercal)\Jac$ is positive stable, and we are done by the corollary above.
\end{proof}

We conclude that SGA converges to an SFP if $\lambda$ is small enough, where `small enough' depends on the additive condition number. 

\subsection{Picking $\sign(\lambda)$}
\label{s:pick_l}

This section explains desiderata $D4$--$D5$ and shows how to pick $\sign(\lambda)$ to speed up convergence towards stable and away from unstable fixed points. In the example below, almost any choice of positive $\lambda$ results in convergence to an unstable equilibrium. The problem arises from the combination of a weak repellor with a strong rotational force. 

\begin{eg}[failure case for $\lambda>0$]\label{eg:get_sign_right}\eod
    Suppose $\epsilon>0$ is small and
  \begin{equation}
    \ell_1(x,y) = -\frac{\epsilon}{2}x^2 -   xy
    \,\text{ and }\,
    \ell_2(x,y) = -\frac{\epsilon}{2}y^2 +   xy
  \end{equation}
  with an unstable equilibrium at $(0,0)$. The dynamics are
  \begin{equation}
    \bxi = \epsilon\cdot \left(\begin{matrix}
      -x \\ -y 
    \end{matrix}\right)
    + 
    \left(\begin{matrix}
      -y \\ x
    \end{matrix}\right)
    \quad\text{with}\quad
    \bA = \left(\begin{matrix}
        0 & -1 \\
        1 & 0
    \end{matrix}\right)
  \end{equation}
  and
  \begin{equation}
      \bA^\intercal\bxi = \left(\begin{matrix}
          x \\ y
      \end{matrix}\right)
      +\epsilon\left(\begin{matrix}
          -y \\ x
      \end{matrix}\right)
  \end{equation}
  Finally observe that
  \begin{equation}
    \bxi + \lambda\cdot\bA^\intercal\bxi
    =
    (\lambda-\epsilon) \cdot\left(\begin{matrix}
        x \\ y
    \end{matrix}\right)
    + (1+\epsilon\lambda)\cdot\left(\begin{matrix}
        - y \\ x
    \end{matrix}\right)
  \end{equation}
  which converges to the unstable equilibrium if $\lambda>\epsilon$.
\end{eg}

We now show how to pick the sign of $\lambda$ to avoid unstable equilibria. First, observe that $\langle\bxi, \grad\cH\rangle = \bxi^\intercal (\bS+\bA)^\intercal \bxi = \bxi^\intercal \bS\bxi$. It follows that for $\bxi\neq0$:
\begin{equation}
    \label{eq:stable_criteria}
    \begin{cases}
    \text{if} \quad \bS\succeq 0 
    & \text{ then } \langle\bxi, \grad\cH\rangle \geq 0; \\
    \text{if} \quad \bS\prec 0 
    & \text{ then } \langle\bxi, \grad\cH\rangle < 0.
    \end{cases}
\end{equation}
A criterion to probe the positive/negative definiteness of $\bS$ is thus to check the sign of $\langle\bxi, \grad\cH\rangle$. The dot product can take any value if $\bS$ is neither positive nor negative (semi-)definite. The behavior near saddle points will be explored in Section \ref{s:strict_saddles}.


\begin{figure}[t]  
    {\center
    \includegraphics[width=.85\textwidth]{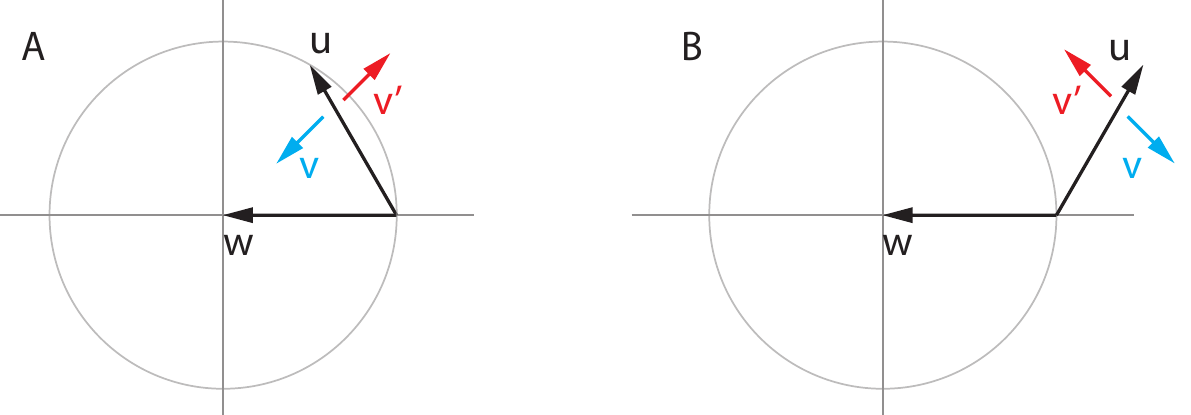}}
    \caption{\emph{Infinitesimal alignment} between $\ut + \lambda\vt$ and $\wt$ is positive (cyan) when small positive $\lambda$ either: \textbf{(A)} pulls $\ut$ toward $\wt$, if $\wt$ and $\ut$ have angle $<90^\circ$; or \textbf{(B)} pushes $\ut$ away from $\wt$ if their angle is $>90^\circ$. Conversely, the infinitesimal alignment is negative (red) when small positive $\lambda$ either: \textbf{(A)} pushes $\ut$ away from $\wt$ when their angle is acute or \textbf{(B)} pulls $\ut$  toward $\wt$ when their angle is obtuse.}
    \label{f:alignment}
\end{figure}

Recall that desiderata $D4$ requires that, if $\bxi$ points at a stable equilibrium then we require that $\bxi_\lambda$ points \emph{more} towards the equilibrium (i.e. has smaller angle). Conversely, desiderata $D5$ requires that, if $\bxi$ points away then the adjustment should point \emph{further} away. More formally,

\begin{defn}
    Let $\ut$ and $\vt$ be two vectors. The \textbf{infinitesimal alignment} of $\bxi_\lambda:= \ut+\lambda\cdot \vt$ with a third vector $\wt$ is
    \begin{equation}
        \algn(\bxi_\lambda, \wt) := \frac{d}{d\lambda}
        \left\{\cos^2 \theta_\lambda\right\}_{|\lambda=0}
        \,\text{ for }\,
        \theta_\lambda := \theta(\bxi_\lambda, \wt).
    \end{equation}
\end{defn}
If $\ut$ and $\wt$ point the same way, $\ut^\intercal\wt>0$, then $\algn>0$ when $\vt$ bends $\ut$ further toward $\wt$, see  figure~\ref{f:alignment}A. Otherwise $\algn>0$ when $\vt$ bends $\ut$ away from $\wt$, see figure~\ref{f:alignment}B.

\begin{algorithm}[tb]
   \caption{Symplectic Gradient Adjustment}
   \label{alg:sga}
\begin{algorithmic}
   \STATE {\bfseries Input:} losses ${\mathcal L} = \{\ell_i\}_{i=1}^n$, weights ${\mathcal W} = \{\wt_i\}_{i=1}^n$
   \STATE $\bxi \leftarrow \big[\texttt{gradient}(\ell_i,\wt_i) \texttt{ for } (\ell_i,\wt_i) \in ({\mathcal L},{\mathcal W})\big]$
    \STATE $\bA^\intercal\bxi \leftarrow \texttt{get\_sym\_adj}({\mathcal L},{\mathcal W})$
    $\quad\quad\quad\,$ // appendix~\ref{s:code}
   \IF{$\texttt{align}$}
   \STATE $\grad\cH \leftarrow \big[\texttt{gradient}(\frac{1}{2} \|\bxi\|^2,\wt) \texttt{ for }\wt \in{\mathcal W})\big]$
\STATE $\lambda  \leftarrow \texttt{sign}\Big(\frac{1}{d}\langle\bxi,\grad \cH\rangle\langle \bA^\intercal \bxi, \grad \cH\rangle
    + \epsilon\Big)
    \quad\text{ // }\epsilon=\frac{1}{10}$
    \ELSE
    \STATE $\lambda\leftarrow 1$
    \ENDIF
   \STATE {\bfseries Output: } $\bxi + \lambda\cdot \bA^\intercal\bxi$
   $\quad\quad\quad$ // plug into any optimizer
\end{algorithmic}
\end{algorithm}

The following lemma allows us to rewrite the infinitesimal alignment in terms of known (computable) quantities, from which we can deduce the correct choice of $\lambda$.

\begin{lem}\label{l:sign_align}
    When $\bxi_\lambda$ is the symplectic gradient adjustment, 
    \begin{equation}
        \sign\Big(\algn(\bxi_\lambda, \grad\cH)\Big)
        = \sign\Big(\langle\bxi,\grad \cH\rangle\cdot \langle \bA^\intercal \bxi, \grad \cH\rangle
    \Big).
    \end{equation}
\end{lem}

\begin{proof}
    Observe that
    \begin{align}
        \cos^2\theta_\lambda 
        = \left(\frac{\langle\bxi_\lambda, \grad\cH\rangle}{\|\bxi_\lambda\|\cdot\|\grad \cH\|}\right)^2
         = \frac{\langle\bxi,\grad \cH\rangle +2\lambda\langle\bxi,\grad \cH\rangle\langle \bA^\intercal\bxi, \grad \cH\rangle +     O(\lambda^2)}
        {\big(\|\bxi\|^2 + O(\lambda^2)\big)\cdot\|\grad\cH\|^2}
    \end{align}
    where the denominator has no linear term in $\lambda$ because $\bxi\perp \bA^\intercal\bxi$. It follows that the sign of the infinitesimal alignment is
    \begin{equation}
        \sign\left\{\frac{d}{d\lambda}\cos^2\theta_\lambda\right\}
        = \sign\Big\{\langle\bxi,\grad \cH\rangle\langle \bA^\intercal\bxi, \grad \cH\rangle\Big\}
    \end{equation}
    as required.
\end{proof}

Intuitively, computing the sign of $\langle\bxi,\grad\cH\rangle$ provides a check for stable and unstable fixed points. Computing the sign of $\langle\bA^\intercal\bxi,\grad\cH\rangle$ checks whether the adjustment term points towards or away from the nearby fixed point. Putting the two checks together yields a prescription for the sign of $\lambda$, as follows.

\begin{prop}\label{p:lambda_sign}
    Desiderata $D4$--$D5$ are satisfied for $\lambda$ such that  $\lambda\cdot \langle\bxi,\grad \cH\rangle\cdot \langle \bA^\intercal \bxi, \grad \cH\rangle\geq0$.
\end{prop}

\begin{proof}
    If we are in a neighborhood of a stable fixed point then $\langle\bxi,\grad\cH\rangle\geq 0$. It follows by lemma~\ref{l:sign_align} that $\sign \Big(\algn(\bxi_\lambda),\grad\cH)\Big) = \sign\Big(\langle\bA^\intercal\bxi,\grad\cH\rangle\Big)$ and so choosing $\sign(\lambda)=\sign\Big(\langle\bA^\intercal\bxi,\grad\cH\rangle\Big)$ leads to the angle between $\bxi_\lambda$ and $\grad\cH$ being smaller than the angle between $\bxi$ and $\grad\cH$, satisfying desideratum $D4$. The proof for the unstable case is similar.
\end{proof}

\paragraph{Alignment and convergence rates.}
Gradient descent is also known as the method of steepest descent. In general games, however, $\bxi$ does not follow the steepest path to fixed points due to the `rotational force', which forces lower learning rates and slows down convergence.

The following lemma provides some intuition about alignment. The idea is that, the smaller the cosine between the `correct direction' $\wt$ and the `update direction' $\bxi$, the smaller the learning rate needs to be for the update to stay in a unit ball, see figure~\ref{f:lemma}.
\begin{lem}[alignment lemma]\eod
  If $\wt$ and $\bxi$ are unit vectors with $0<\wt^\intercal\bxi$ then $\|\wt - \eta\cdot \bxi\|\leq 1$ for $0\leq \eta\leq 2\wt^\intercal\bxi = 2\cos \theta(\wt, \bxi)$. In other words, ensuring that $\wt-\eta\bxi$ is closer to the origin than $\wt$ requires smaller learning rates $\eta$ as the angle between $\wt$ and $\bxi$ gets larger. 
\end{lem}
\begin{proof}
    Check $\|\wt-\eta\cdot \bxi\|^2 = 1 + \eta^2 -2\eta\cdot \wt^\intercal\bxi \leq 1$ iff $\eta^2\leq 2\eta\cdot\wt^\intercal\bxi$. The result follows.
\end{proof}

\begin{figure}[t]  
    \center
    \includegraphics[width=.9\textwidth]{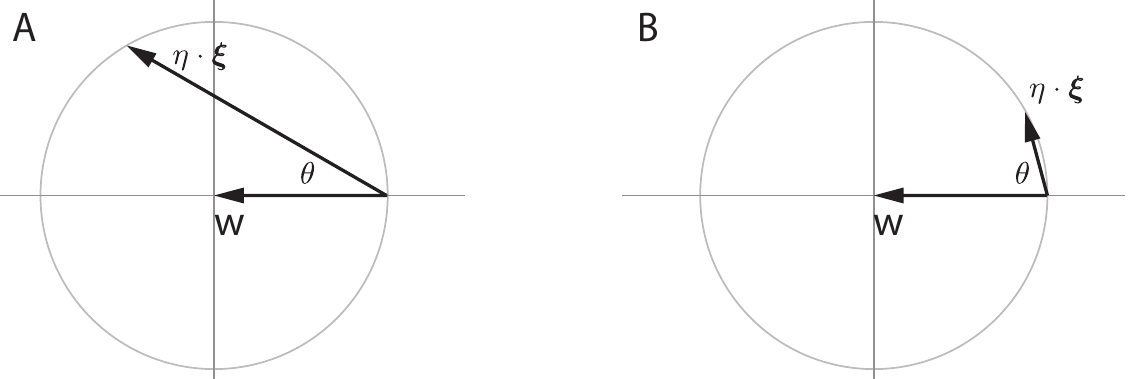}
    \caption{\emph{Alignment and learning rates.} The larger $\cos\theta$, the larger the learning rate $\eta$ that can be applied to unit vector $\bxi$ without $\wt+\eta\cdot\bxi$ leaving the unit circle.}
    \label{f:lemma}
\end{figure}

The next lemma is a standard technical result from the convex optimization literature.
\begin{lem}\label{lem:nesterov}
    Let $f:\bR^d\rightarrow \bR$ be a convex Lipschitz smooth function satisfying $\|\grad f(\y) - \grad f(\x)\|\leq L\cdot \|\y - \x\|$ for all $\x,\y\in\bR^d$. Then 
    \begin{equation}
        \left|f(\y) - f(\x) -\langle\grad f(\x),\y-\x\rangle\right|
        \leq \frac{L}{2} \cdot \|\y-\x\|^2
    \end{equation}
    for all $\x,\y\in\bR^d$.
\end{lem}

\begin{proof}
    See \citet{nesterov:04}.
\end{proof}

Finally, we show that increasing alignment helps speed convergence: 

\begin{thm}\label{t:cosine}
    Suppose $f$ is convex and Lipschitz smooth with  $\|\grad f(\x)-\grad f(\y)\| \leq L\cdot\|\x-\y\|$. Let $\wt_{t+1} = \wt_t - \eta\cdot \vt$ where $\|\vt\|=\|\grad f(\wt_t)\|$. Then the optimal step size is $\eta^* = \frac{\cos\theta}{L}$ where $\theta := \theta(\grad f(\wt_t), \vt)$, with 
    \begin{equation}
        f(\wt_{t+1}) \leq f(\wt_t) - \frac{\cos^2\theta}{2L}\cdot \|\grad f(\wt_t)\|^2.
    \end{equation}
\end{thm}

The proof of Theorem~\ref{t:cosine} adapts lemma~\ref{lem:nesterov} to handle the angle arising from the `rotational force'.

\begin{proof}
    By the lemma~\ref{lem:nesterov},
    \begin{align}
        f(\y) & \leq f(\x) + \langle\grad f(\x),\y-\x\rangle + \frac{L}{2} \|\y-\x\|^2 \\
        & = f(\x) - \eta\cdot\langle \grad f,\bxi\rangle + \eta^2\frac{L}{2}\cdot \|\bxi\|^2 \\
        & = f(\x) - \eta\cdot\langle\grad f,\bxi\rangle + \eta^2\frac{L}{2}\cdot \|\grad f\|^2  \\ 
        & = f(\x) - \eta(\alpha - \frac{\eta}{2}L)\cdot\|\grad f\|^2
    \end{align}
    where $\alpha:=\cos\theta$. Solve
    \begin{equation}
        \min_\eta \Delta(\eta) = \min_\eta\left\{ -\eta(\alpha-\frac{\eta}{2}L)\right\}
    \end{equation}
    to obtain $\eta^* = \frac{\alpha}{L}$ and $\Delta(\eta^*) = -\frac{\alpha^2}{2}L$ as required.
\end{proof}

Increasing the cosine with the steepest direction improves convergence. The alignment computation in algorithm~\ref{alg:sga} chooses $\lambda$ to be positive or negative such that $\bxi_\lambda$ is bent towards stable (increasing the cosine) and away from unstable fixed points. Adding a small $\epsilon>0$ to the computation introduces a weak bias towards stable fixed points.

\begin{figure}[t]
    \center
    \includegraphics[width=.9\textwidth]{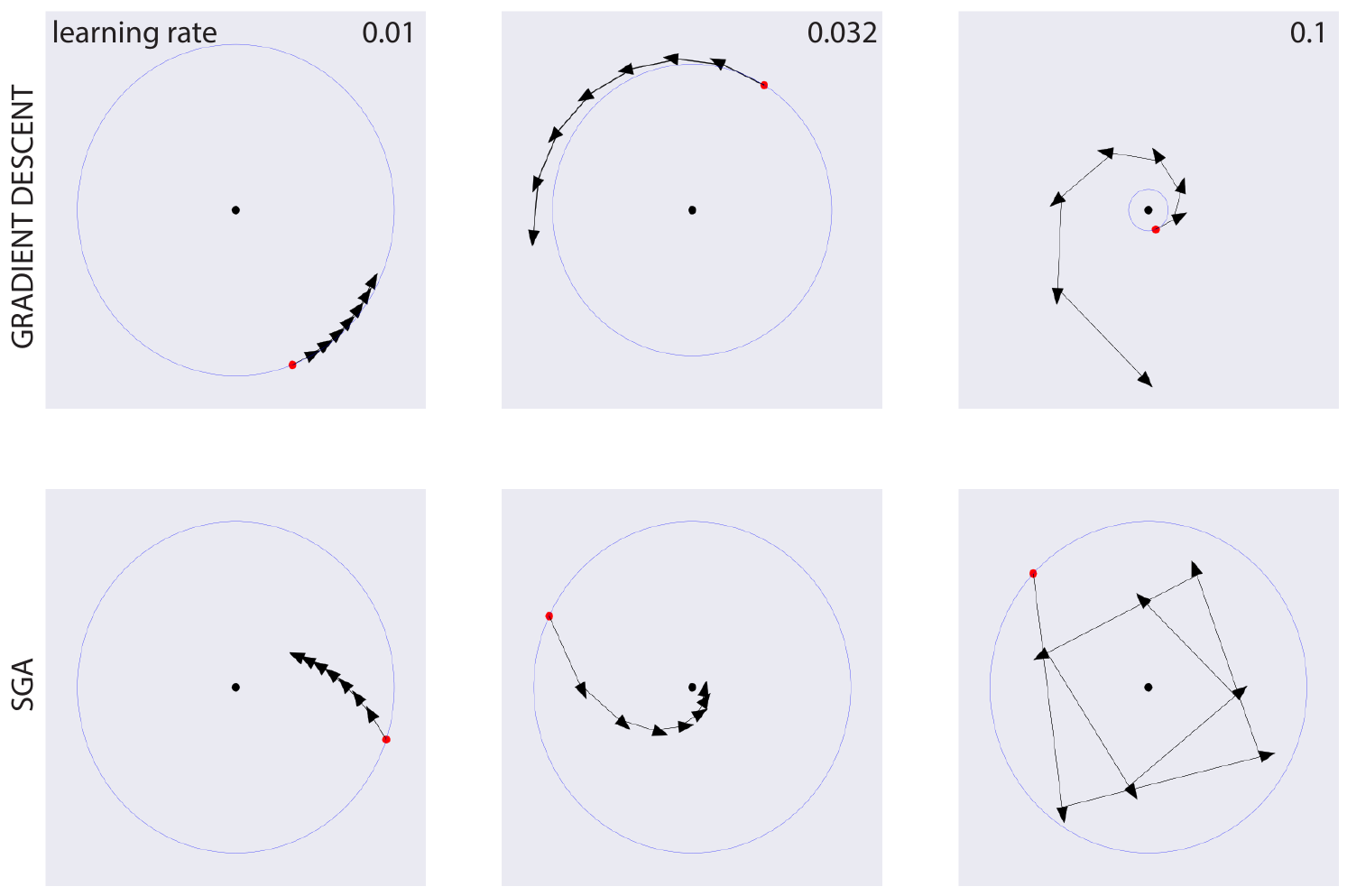}
    \vspace{-5mm}
    \caption{
    SGA allows faster and more robust convergence to stable fixed points than vanilla gradient descent in the presence of `rotational forces', by bending the direction of descent towards the fixed point. Note the gradient descent diverges extremely rapidly in the top-right panel, which has a different scale from the other panels. 
    }
    \label{f:learning_rates}
\end{figure}

\subsection{Aligned Consensus Optimization}
\label{s:align_con}

The stability criterion in \eqref{eq:stable_criteria} also provides a simple way to prevent consensus optimization from converging to unstable equilibria. \textbf{Aligned consensus optimization} is
\begin{equation}
    \label{eq:aco}
    \bxi + |\lambda|\cdot \sign\Big(\langle\bxi,\grad\cH\rangle\Big)\cdot \Jac^\intercal\bxi,
\end{equation}
where in practice we set $\lambda=1$. Aligned consensus optimization satisfies desiderata \emph{D3--D5}. However, it behaves strangely in potential games. Multiplying by the Jacobian is the `inverse' of Newton's method since for potential games the Jacobian of $\bxi$ is the Hessian of the potential function. Multiplying by the Hessian increases the gap between small and large eigenvalues, increasing the (usual, multiplicative) condition number and slows down convergence. Nevertheless, consensus optimization works well in GANs \citep{mescheder:17}, and aligned consensus may improve performance, see experiments below.

Dropping the first term $\bxi$ from \eqref{eq:aco} yields a simpler update that also satisfies $D3$--$D5$. However, the resulting algorithm performs poorly in experiments (not shown), perhaps because it is attracted to saddles. 

\subsection{Avoiding Strict Saddles}
\label{s:strict_saddles}

How does SGA behave near saddles? We show that Symplectic Gradient Adjustment locally avoids strict saddles, provided that $\lambda$ and $\alpha$ are small and parameters are initialized with (arbitrarily small) noise. More precisely, let $\bF(\wt) = \wt - \alpha \bxi_\lambda(\wt)$ be the iterative optimization procedure given by SGA. Then every strict saddle $\wt^*$ has a neighbourhood $U$ such that $\{ \wt \in U \mid \bF^n(\wt) \to \wt^* \text{ as } n \to \infty \}$ has measure zero for small $\alpha > 0$ and $\lambda$.

Intuitively, the Taylor expansion around a strict saddle $\wt^*$ is locally dominated by the Jacobian at $\wt^*$, which has a negative eigenvalue. This prevents convergence to $\wt^*$ for random initializations of $\wt$ near $\wt^*$. The argument is made rigorous using the Stable Manifold Theorem following \citet{lee:17}.

\begin{thm}[Stable Manifold Theorem]\eod
    Let $\wt^*$ be a fixed point for the $C^1$ local diffeomorphism $F : U \rightarrow \R^d$, where $U$ is a neighbourhood of $\wt^*$ in $\R^d$. Let $E^s \oplus E^u$ be the generalized eigenspaces of $\nabla F(\wt^*)$ corresponding to eigenvalues with $\left| \sigma \right| \leq 1$ and $\left| \sigma \right| > 1$ respectively. Then there exists a \textit{local stable center manifold} $W$ with tangent space $E^s$ at $\wt^*$ and a neighbourhood $B$ of $\wt^*$ such that $F(W) \cap B \subset W$ and $\cap_{n=0}^{\infty} F^{-n}(B) \subset W$.
\end{thm}

\begin{proof}
    See \citet{shub:00}.
\end{proof}

It follows that if $\nabla F(\wt^*)$ has at least one eigenvalue $\left| \sigma \right| > 1$ then $E^u$ has dimension at least $1$. Since $W$ has tangent space $E^s$ at $\wt^*$ with codimension at least one, we conclude that $W$ has measure zero. This is central to proving that the set of nearby initial points which converge to a given strict saddle $\wt^*$ has measure zero. Since $\wt$ is initialized randomly, the following theorem is obtained.

\begin{figure}[t]
    \center
    \includegraphics[width=.9\textwidth]{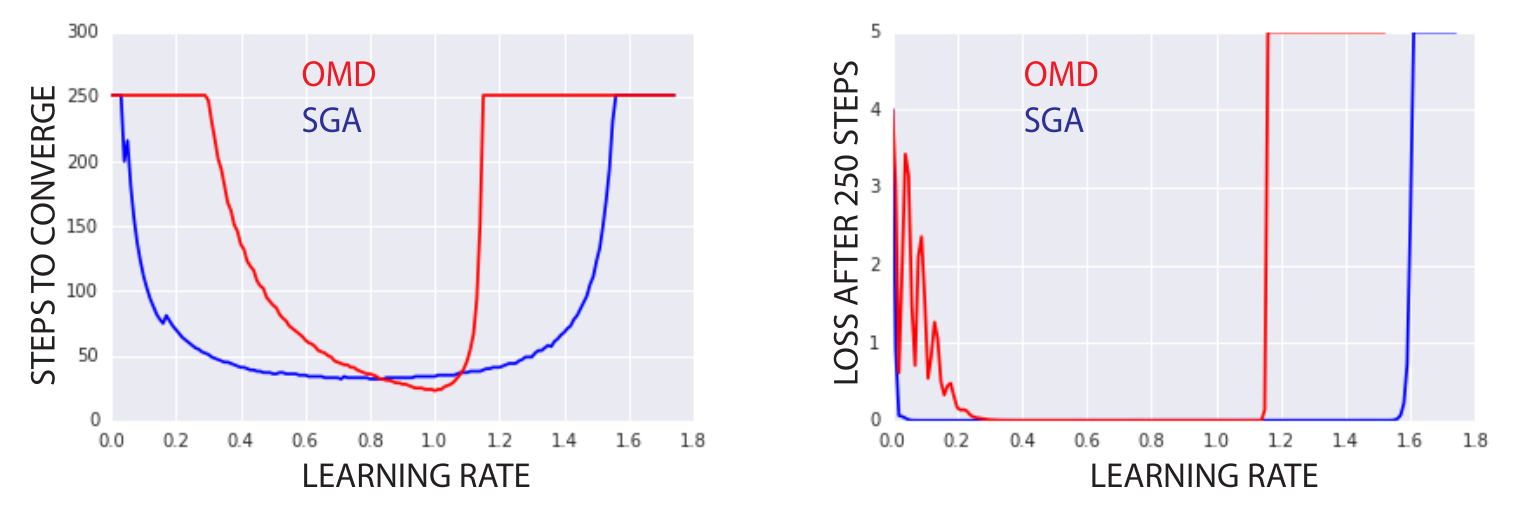}
    \caption{
    Comparison of SGA with optimistic mirror descent. The plots sweep over learning rates in range $[0.01, 1.75]$, with $\lambda=1$ throughout for SGA. \textbf{(Left):} iterations to convergence, with maximum value of 250 after which the run was interrupted. \textbf{(Right):} average absolute value of losses over the last 10 iterations, 240-250, with a cutoff at 5.
    }
    \label{f:omd}
\end{figure}

\begin{thm}\label{thm:avoid_saddle}
    SGA locally avoids strict saddles almost surely, for $\al > 0$ and $\la$ small.
\end{thm}

\begin{proof}
Let $\wt^*$ a strict saddle and recall that SGA is given by
\[ F(\wt) = \wt - \al(\indicator-\al \Jac)\bxi(\wt) \,. \]
All terms involved are continuously differentiable and we have
\[ \nabla F(\wt^*) = \indicator-\al(\indicator-\al \Jac)\Jac(\wt^*) \]
by assumption that $\bxi(\wt^*) = 0$. Since all terms except $\indicator$ are of order at least $\al$, $\nabla F(\wt^*)$ is invertible for all $\al$ sufficiently small. By the inverse function theorem, there exists a neighbourhood $U$ of $\wt^*$ such that $F$ is has a continuously differentiable inverse on $U$. Hence $F$ restricted to $U$ is a $C^1$ diffeomorphism with fixed point $\wt^*$.

By definition of strict saddles, $\Jac(\wt^*)$ has an eigenvalue with negative real part. It follows by continuity that $(\indicator-\al \Jac)\Jac(\wt^*)$ also has an eigenvalue $a+ib$ with $a < 0$ for $\al$ sufficiently small. Finally,
\[ \nabla F(\wt^*) = \indicator-\al(\indicator-\al \Jac)\Jac(\wt^*) \]
has an eigenvalue $\sigma = 1-\al a - i \al b$ with
\[ \left| \sigma \right| = 1-2\al a + \al^2(a^2+b^2) \geq 1-2\al a > 1 \,. \]
It follows that $E^s$ has codimension at least one, implying in turn that the local stable set $W$ has measure zero. We can now prove that
\[ Z = \{ \wt \in U \mid \lim_{n \to \infty} F^n(\wt) = \wt^* \} \]
has measure zero, or in other words, that local convergence to $\wt^*$ occurs with zero probability. Let $B$ the neighbourhood guaranteed by the Stable Manifold Theorem, and take any $\wt \in Z$. By definition of convergence there exists $N \in \N$ such that $F^{N+n}(\wt) \in B$ for all $n \in \N$, so that
\[ F^N(\wt) \in \cap_{n \in \N}^{\infty} F^{-n}(B) \subset W\]
by the Stable Manifold Theorem. This implies that $\wt \in F^{-N}(W)$, and by extension $\wt \in \cup_{n \in \N} F^{-n}(W)$. Since $\wt$ was arbitrary, we obtain the inclusion
\[ Z \subseteq \cup_{n \in \N} F^{-n}(W) \,. \]
Now $F^{-1}$ is $C^1$, hence locally Lipschitz and thus preserves sets of measure zero, so that $F^{-n}(W)$ has measure zero for each $n$. Countable unions of measure zero sets are still measure zero, so we conclude that $Z$ also has measure zero. In other words, SGA converges to $\wt^*$ with zero probability upon random initialization of $\wt$ in $U$.
\end{proof}

Unlike stable and unstable fixed points, it is unclear how to avoid strict saddles using only alignment, that is, independently from the size of $\lambda$.

\section{Experiments}
\label{s:exp}

We compare SGA with simultaneous gradient descent, optimistic mirror descent \citep{daskalakis:18} and consensus optimization \citep{mescheder:17} in basic settings.

\begin{figure*}[t]
    \center
    \includegraphics[width=.75\textwidth]{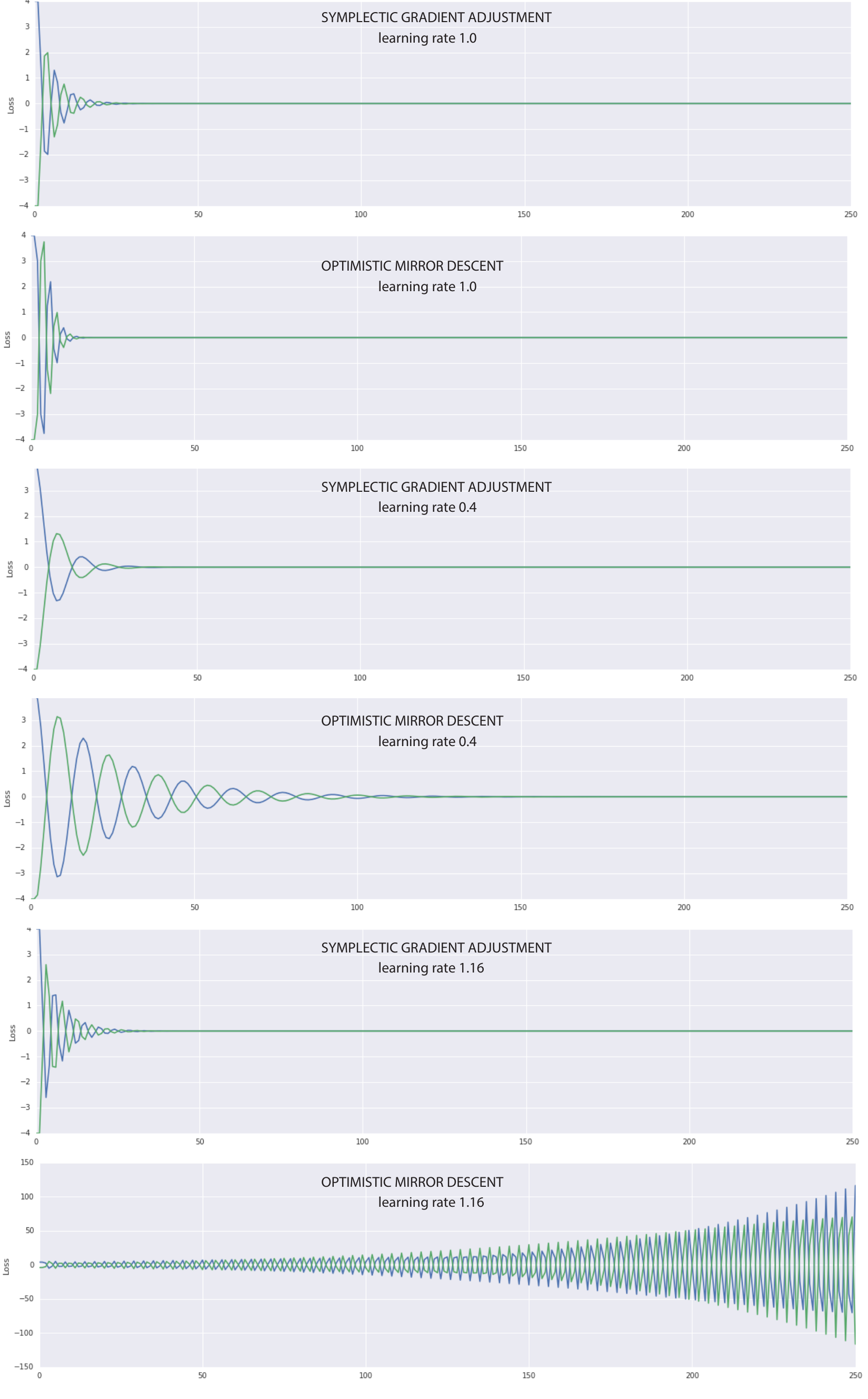}
    \caption{
    Individual runs on zero-sum bimatrix game in section~\ref{s:zsbmg}.
    }
    \label{f:zsbmg}
\end{figure*}
\subsection{Learning rates and alignment}

We investigate the effect of SGA when a weak attractor is coupled to a strong rotational force:
\begin{equation}
  \ell_1(x,y) = \frac{1}{2}x^2 + 10xy
  \quad\text{and}\quad
  \ell_2(x,y) = \frac{1}{2}y^2 - 10xy
\end{equation}
Gradient descent is extremely sensitive to the choice of learning rate $\eta$, top row of figure~\ref{f:learning_rates}. As $\eta$ increases through $\{0.01, 0.032, 0.1\}$ gradient descent goes from converging extremely slowly, to diverging slowly, to diverging rapidly. SGA yields faster, more robust convergence. SGA converges faster with learning rates $\eta=0.01$ and $\eta=0.032$, and only starts overshooting the fixed point for $\eta=0.1$.

\subsection{Basic adversarial games}
\label{s:zsbmg}

Optimistic mirror descent is a family of algorithms that has nice convergence properties in games \citep{rakhlin:13, syrgkanis:15}. In the special case of optimistic \emph{gradient} descent the updates are
\begin{equation}
    \wt_{t+1} \leftarrow \wt_t - \eta\cdot\bxi_t -\eta\cdot (\bxi_t - \bxi_{t-1}).
\end{equation}
Figure~\ref{f:omd} compares SGA with optimistic gradient descent (OMD) on a zero-sum bimatrix game with $\ell_{1/2}(\wt_1,\wt_2) = \pm\wt_1^\intercal\wt_2$. The example is modified from \citet{daskalakis:18} who also consider a linear offset that makes no difference. A run is taken to have converged if the average absolute value of losses on the last 10 iterations is $<0.01$; we end each experiment after 250 steps. 

The left panel shows the number of steps to convergence (when convergence occurs) over a range of learning rates. OMD's peak performance is better than SGA, where the red curve dips below the blue. Howwever, we find that SGA converges -- and does so faster -- for a much wider range of learning rates. OMD diverges for learning rates not in the range [0.3, 1.2]. Simultaneous gradient descent oscillates without converging (not shown). The right panel shows the average performance of OMD and SGA on the last 10 steps. Once again, here SGA consistently performs better over a wider range of learning rates. Individual runs are shown in figure~\ref{f:zsbmg}.

\paragraph{OMD and SGA on a four-player game.}
Figure~\ref{f:4player} shows time to convergence (using the same convergence criterion as above) for optimistic mirror descent and SGA. The games are constructed with four players, each of which controls one parameter. The losses are
\begin{align}
    \ell_1(w, x, y, z) & = \frac{\epsilon}{2}w^2 + wx + wy + wz \\
    \ell_2(w, x, y, z) & = -wx + \frac{\epsilon}{2}x^2 + xy + xz
\end{align}
\begin{align}
    \ell_3(w, x, y, z) & = -wy - xy + \frac{\epsilon}{2}y^2 + yz \\
    \ell_4(w, x, y, z) & = -wz - xz  - yz + \frac{\epsilon}{2}z^2,
\end{align}
where $\epsilon=\frac{1}{100}$ in the left panel and $\epsilon=0$ in the right panel. The antisymmetric component of the game Jacobian is 
\begin{equation}
    \bA = \left(\begin{matrix}
    0 & 1 & 1 & 1\\
    -1 & 0 & 1 & 1\\
    -1 & -1 & 0 & 1\\
    -1 & -1 & -1 & 0
    \end{matrix}\right)
\end{equation}
and the symmetric component is 
\begin{equation}
    \bS = \epsilon\cdot\left(\begin{matrix}
    1 & 0 & 0 & 0\\
    0 & 1 & 0 & 0\\
    0 & 0 & 1 & 0\\
    0 & 0 & 0 & 1
    \end{matrix}\right).
\end{equation}
OMD converges considerably slower than SGA across the full range of learning rates. It also diverges for learning rates $>0.22$. In contrast, SGA converges more quickly and robustly.

\begin{figure}[t]
    \center
    \includegraphics[width=.9\textwidth]{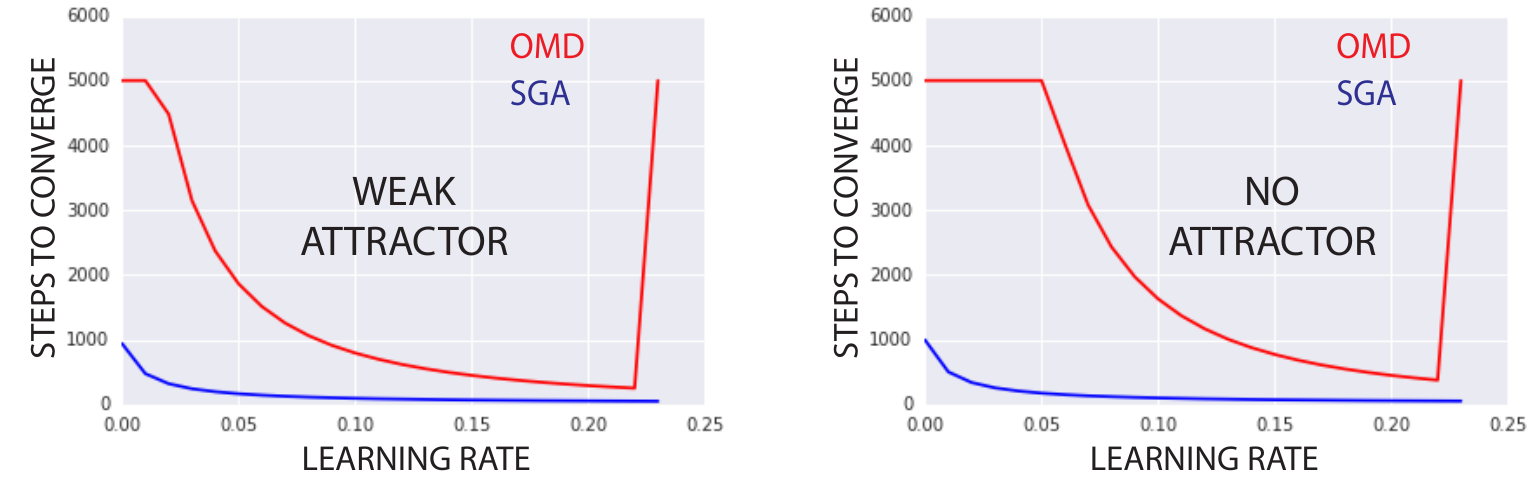}
    \caption{
    Time to convergence of OMD and SGA on two 4-player games. Times are cutoff after 5000 iterations. \textbf{Left panel:} Weakly positive definite $\bS$ with $\epsilon=\frac{1}{100}$. \textbf{Right panel:} Symmetric component is identically zero.
    }
    \label{f:4player}
\end{figure}

\subsection{Learning a two-dimensional mixture of Gaussians}

We apply SGA to a basic Generative Adversarial Network setup adapted from \citet{metz:17}. Data is sampled from a highly multimodal distribution designed to probe the tendency of GANs to collapse onto a subset of modes during training. The distribution is a mixture of 16 Gaussians arranged in a $4\times 4$ grid. Figure~\ref{f:ground_truth} shows the probability distribution that is sampled to train the generator and discriminator. The generator and discriminator networks both have 6 ReLU layers of 384 neurons. The generator has two output neurons; the discriminator has one.  

Figure~\ref{f:gans} shows results after $\{2000, 4000, 6000, 8000\}$ iterations. The networks are trained under RMSProp. Learning rates were chosen by visual inspection of grid search results at iteration 8000.  More precisely, grid search was over learning rates $\{$1e-5, 2e-5,5e-5, 8e-5, 1e-4, 2e-4, 5e-4$\}$ and then a more refined linear search over $[$8e-5, 2e-4$]$. Simultaneous gradient descent and SGA are shown in the figure. 

The last two rows of figure~\ref{f:gans} show the performance of consensus optimization without and with alignment. Introducing alignment slightly improves speed of convergence (second column) and final result (fourth column), although intermediate results in third column are ambiguous.

Simultaneous gradient descent exhibits mode collapse followed by mode hopping in later iterations (not shown). Mode hopping is analogous to the cycles in example~\ref{eg:basic_eg}. Unaligned SGA converges to the correct distribution; alignment speeds up convergence slightly. Consensus optimization performs similarly in this GAN example. However, consensus optimization can converge to local maxima even in potential games, recall example~\ref{eg:con_fail}.

\begin{figure}[t]
    \center
    \includegraphics[width=.30\textwidth]{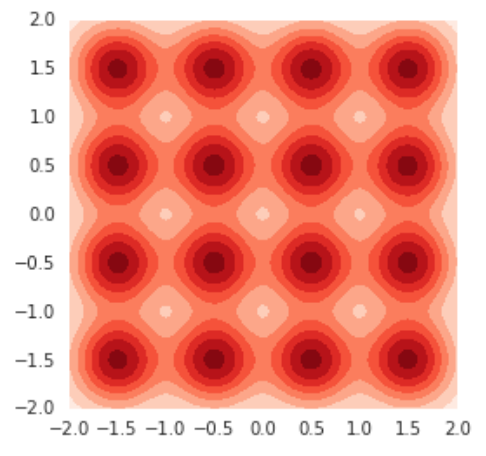}
    \caption{
    Ground truth for GAN experiments on a two-dimensional mixture of 16 Gaussians. 
    }
    \label{f:ground_truth}
\end{figure}

\begin{figure*}[t]
    \center
    \includegraphics[width=.8\textwidth]{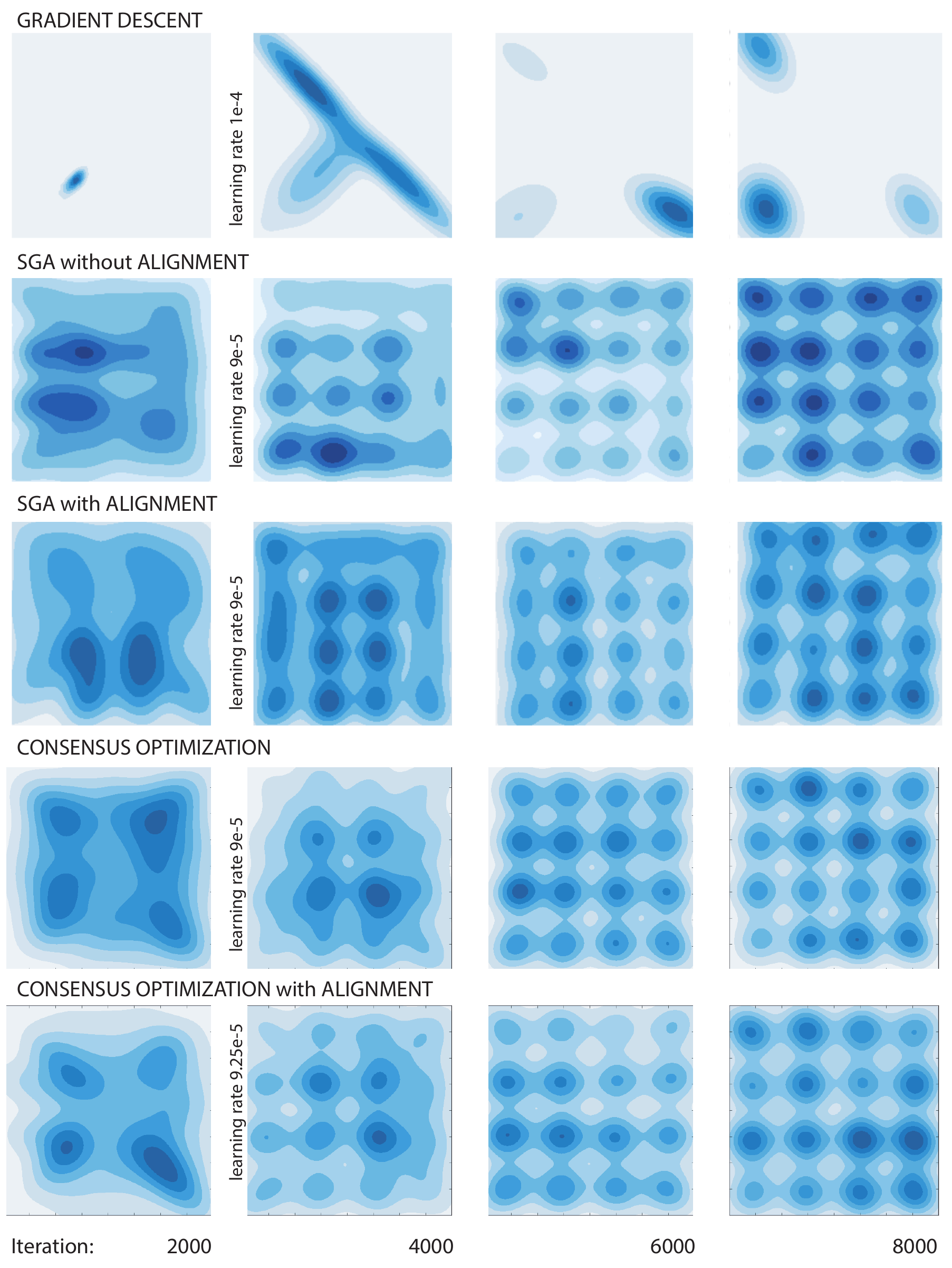}
    \caption{
      \textbf{First row:} Simultaneous gradient descent suffers from mode collapse and in later iterations (not shown) mode hopping. 
      \textbf{Second and third rows:} vanilla SGA converges smoothly to the ground truth (figure~\ref{f:ground_truth}). SGA with alignment converges slightly faster.
        \textbf{Fourth and fifth rows:} Consensus optimization without and with alignment.
    }
    \label{f:gans}
\end{figure*}

\subsection{Learning a high-dimensional unimodal Gaussian}

Mode collapse is a well-known phenomenon in GANs. A more subtle phenomenon, termed boundary distortion, was identified in \citet{santurkar:18}. Boundary distortion is a form of covariate shift where the generator fails to model the true data distribution. 

Santurkar \emph{et al} demonstrate boundary distortion using data sampled from a 75-dimensional unimodal Gaussian with spherical covariate matrix. Mode collapse is not a problem in this setting because the data distribution is unimodal. Nevertheless, they show that vanilla GANs fail to learn most of the spectrum of the covariate matrix. 

Figure~\ref{f:hdg} reproduces their result. Panel A shows the ground truth: all 75 eigenvalues are equal to 1.0. Panel B shows the spectrum of the covariance matrix of the data generated by a GAN trained with RMSProp. The GAN concentrates on a single eigenvalue and essentially ignores the remaining 74 eigenvalues. This is similar to, but more extreme than, the empirical results obtained in \citet{santurkar:18}. We emphasize that the problem is not mode collapse, since the data is unimodal (although, it's worth noting that most of the mass of a high-dimensional Gaussian lies on the ``shell'').

Finally, panel C shows the spectrum of the covariance matrix of the data sampled from a GAN trained via SGA. The GAN approximately learns all the eigenvalues, with values ranging between 0.6 and 1.5.

\begin{figure}[t]
    \center
    \includegraphics[width=\textwidth]{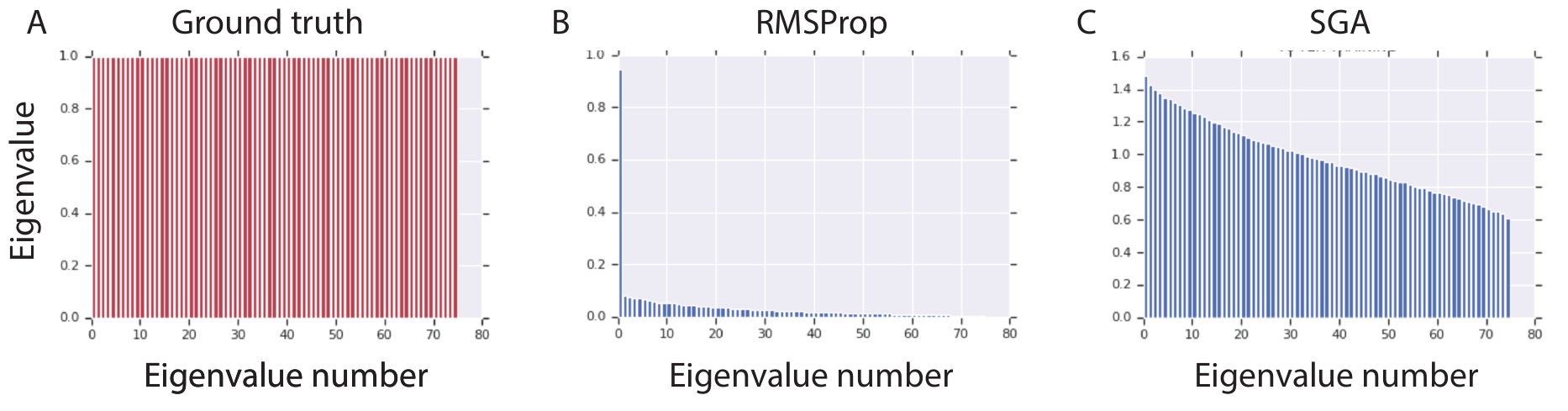}
    \caption{
    \textbf{Panel A:} The ground truth is a 75 dimensional spherical Gaussian whose covariance matrix has all eigenvalues equal to 1.0. 
    \textbf{Panel B:} A vanilla GAN trained with RMSProp approximately learns the first eigenvalue, but essentially ignores all the rest.
    \textbf{Panel C:} Applying SGA results in the GAN approximately learning all 75 eigenvalues, although the range varies from 0.6 to 1.5.
    }
    \label{f:hdg}
\end{figure}

\section{Discussion}

Modern deep learning treats differentiable modules like plug-and-play lego blocks. For this to work, at the very least, we need to know that gradient descent will find local minima. Unfortunately, gradient descent does \emph{not} necessarily find local minima when optimizing multiple interacting objectives. With the recent proliferation of algorithms that optimize more than one loss, it is becoming increasingly urgent to understand and control the dynamics of interacting losses. Although there is interesting recent work on two-player adversarial games such as GANs, there is essentially no work on finding stable fixed points in more general games played by interacting neural nets.

The generalized Helmholtz decomposition provides a powerful new perspective on game dynamics. A key feature is that the analysis is indifferent to the number of players. Instead, it is the interplay between the simultaneous gradient $\bxi$ on the losses and the symmetric and antisymmetric matrices of second-order terms that guides algorithm design and governs the dynamics under gradient adjustments. 

Symplectic gradient adjustment is a straightforward application of the generalized Helmholtz decomposition. It is unlikely that SGA is the best approach to finding stable fixed points. A deeper understanding of the interaction between the potential and Hamiltonian components will lead to more effective algorithms. Reinforcement learning algorithms that optimize multiple objectives are increasingly common, and second-order terms are difficult to estimate in practice. Thus, first-order methods that do not use Jacobian-vector products are of particular interest.

\paragraph{Gamification.}
Finally, it is worth raising a philosophical point. In this paper we are concerned with finding stable fixed points (because, for example, they yield pleasing samples in GANs). We are not concerned with the losses of the players \emph{per se}. The gradient adjustments may lead to a player acting against its own self-interest by increasing its loss. We consider this acceptable insofar as it encourages convergence to a stable fixed point. The players are but a means to an end.

We have argued that stable fixed points are a more useful solution concept than local Nash equilibria for our purposes. However, neither is entirely satisfactory, and the question ``What is the right solution concept for neural games?'' remains open. In fact, it likely has many answers. The intrinsic curiosity module introduced by \citet{pathak:17} plays two objectives against one another to drive agents to search for novel experiences. In this case, converging to a fixed point is precisely what is to be avoided. 

It is remarkable -- to give a few examples sampled from many -- that curiosity, generating photorealistic images, and image-to-image translation \citep{zhu:17} can be formulated as games. What else can games do?

\paragraph{Acknowledgements.}
We thank Guillaume Desjardins and Csaba Szepesvari for useful comments.

{

}

\setcounter{section}{0}
\renewcommand{\thesection}{\Alph{section}}


\vspace{5mm}
\noindent
{\textsf{\textbf{\Large{APPENDIX}}}}

\section{TensorFlow code to compute SGA}
\label{s:code}

Source code is available at \url{https://github.com/deepmind/symplectic-gradient-adjustment}. Since computing the symplectic adjustment is quite simple, we include an explicit description here for completeness.

The code requires a list of $n$ losses, $\mathtt{Ls}$, and a list of variables for the $n$ players, $\mathtt{xs}$. The function $\texttt{fwd\_gradients}$ which implements forward mode auto-differentiation is in the module $\texttt{tf.contrib.kfac.utils}$.
${}\\$

\noindent
\% \textbf{compute Jacobian-vector product $\Jac\vt$}

\noindent
$\mathtt{\texttt{def jac\_vec}(ys, xs, vs):}$\\
 $\texttt{}\hspace{.4cm}\quad\mathtt{\texttt{return fwd\_gradients}(ys, xs,}$
 $\texttt{}\mathtt{\texttt{grad\_xs}=vs, \texttt{stop\_gradients}=xs)}$\\

\noindent
\% \textbf{compute Jacobian$^\intercal$-vector product $\Jac^\intercal\vt$}

\noindent
$\mathtt{\texttt{def jac\_tran\_vec}(ys, xs, vs):}$\\
$\texttt{}\hspace{.4cm}\mathtt{dydxs = \texttt{tf.gradients}(ys, xs, \texttt{grad\_ys}=vs,}$
$\mathtt{\texttt{stop\_gradients}=xs)}$\\
$\texttt{}\hspace{.4cm}\mathtt{\texttt{return }[\texttt{tf.zeros\_like}(x)\texttt{ if }dydx\texttt{ is }None}$\\
$\texttt{}\hspace{1.9cm}\mathtt{\texttt{else }dydx}$
$\texttt{}\mathtt{\texttt{for }(x, dydx)\texttt{ in zip}(xs, dydxs)]}$

${}\\$
\noindent
\% \textbf{compute Symplectic Gradient Adjustment $\bA^\intercal\xi$}

\noindent
$\mathtt{\texttt{def get\_sym\_adj}(Ls, xs):}\\$
$\texttt{}\hspace{.4cm}$\% \textbf{compute game dynamics $\bxi$}\\
$\texttt{}\hspace{.4cm}\mathtt{xi= [\texttt{tf.gradients}(\ell,x)[0]}
\mathtt{\texttt{ for }(\ell,x)\texttt{ in zip}(Ls, xs)]}$
\\
$\texttt{}\hspace{.4cm}\mathtt{J\_xi = \texttt{jac\_vec}(xi,xs,xi)}$\\
$\texttt{}\hspace{.4cm}\mathtt{Jt\_xi = \texttt{jac\_tran\_vec}(xi,xs,xi)}$\\
$\texttt{}\hspace{.4cm}$\% \textbf{compute $\bA^\intercal\bxi = \frac{1}{2}(\Jac^\intercal\bxi-\Jac\bxi)$}
\\
$\texttt{}\hspace{.4cm}\mathtt{At\_xi= [\frac{jt - j}{2}\texttt{ for }(j, jt)\texttt{ in zip}(J\_xi, Jt\_xi)]}\quad\quad$\\
$\texttt{}\hspace{.4cm}\mathtt{\texttt{return }At\_xi}$

\section{Helmholtz, Hamilton, Hodge, and Harmonic games}
\label{s:diff}

This section explains the mathematical connections with the Helmholtz decomposition, symplectic geometry and the Hodge decomposition. The discussion is \emph{not} necessary to understand the main text. It is also not self-contained. The details can be found in textbooks covering differential and symplectic geometry \citep{arnold:89, guillemin:90, bott:95}.

\subsection{The Helmholtz Decomposition}

The classical Helmholtz decomposition states that any vector field $\bxi$ in 3-dimensions is the sum of curl-free (gradient) and divergence-free (infinitesimal rotation) components:
\begin{equation}
    \bxi = \underbrace{\grad\phi}_\text{gradient component} + \underbrace{\text{curl}(\bB)}_\text{rotational component}
    \quad\Big[\text{ curl}(\bullet):=\grad\times(\bullet)\Big]
\end{equation}
We explain the link between $\text{curl}$ and the antisymmetric component of the game Jacobian. Recall that gradients of functions are actually differential 1-forms, not vector fields. Differential 1-forms and vector fields on a manifold are canonically isomorphic once a Riemannian metric has been chosen. In our case, we are implicitly using the Euclidean metric. The antisymmetric matrix $\bA$ is the differential 2-form obtained by applying the exterior derivative $d$ to the 1-form $\bxi$.

In 3-dimensions, the Hodge star operator is an isormorphism from differential 2-forms to vector fields, and the curl can be reformulated as $\text{curl}(\bullet)=*d(\bullet)$. In claiming $\bA$ is analogous to $\text{curl}$, we are simply dropping the Hodge-star operator.

Finally, recall that the Lie algebra of infinitesimal rotations in $d$-dimensions is given by antisymmetric matrices. When $d=3$, the Lie algebra can be represented as vectors (three numbers specify a $3\times3$ antisymmetric matrix) with the $\times$-product as Lie bracket. In general, the antisymmetric matrix $\bA$ captures the infinitesimal tendency of $\bxi$ to rotate at each point in the parameter space.

\subsection{Hamiltonian Mechanics}

A symplectic form $\omega$ is a closed nondegenerate differential 2-form. Given a manifold with a symplectic form, a vector field $\bxi$ is \textbf{Hamiltonian vector field} if there exists a function $\cH:M\rightarrow \bR$ satisfying
\begin{equation}
    \label{eq:hamiltonian_vector_field}
    \omega(\bxi, \bullet) = d\cH(\bullet) = \langle\grad \cH,\bullet\rangle.
\end{equation}
The function is then referred to as the Hamiltonian function of the vector field. In our case, the antisymmetric matrix $\bA$ is a closed 2-form because $\bA = d\bxi$ and $d\circ d=0$. It may however be degenerate. It is therefore a presymplectic form \citep{bottacin:05}.

Setting $\omega=\bA$, equation~\eqref{eq:hamiltonian_vector_field} can be rewritten in our notation as
\begin{equation}
    \underbrace{\omega(\bxi,\bullet)}_{\bA^\intercal\bxi} = \underbrace{d\cH(\bullet)}_{\grad\cH},
\end{equation}
justifying the terminology `Hamiltonian'.

\subsection{The Hodge Decomposition}

The exterior derivative $d_k:\Omega^k(M)\rightarrow \Omega^{k+1}(M)$ is a linear operator that takes differential $k$-forms on a manifold $M$, $\Omega^k(M)$, to differential $k+1$-forms, $\Omega^{k+1}(M)$. In the case $k=0$, the exterior derivative is the gradient, which takes 0-forms (that is, functions) to 1-forms. Given a Riemannian metric, the adjoint of the exterior derivative $\delta$ goes in the opposite direction. Hodge's theorem  states that $k$-forms on a compact manifold decompose into a direct sum over three types:
\begin{equation}
    \Omega^k(M) = d\Omega^{k-1}(M) \oplus \text{Harmonic}^k(M) \oplus \delta \Omega^{k+1}(M).
\end{equation}
Setting $k=1$, we recover a decomposition that closely resembles the generalized Helmholtz decomposition: 
\begin{equation}
    \underbrace{\Omega^1(M)}_\text{1-forms} = \underbrace{d\Omega^{0}(M)}_\text{gradients of functions} \oplus\text{Harm}^k(M) \oplus  \underbrace{\delta \Omega^{2}(M)}_\text{antisymmetric component}
\end{equation}
The harmonic component is isomorphic to the de Rham cohomology of the manifold -- which is zero when $k=1$ and $M=\bR^n$. 

Unfortunately, the Hodge decomposition does not straightforwardly apply to the case when $M=\bR^n$, since $\bR^n$ is not compact. It is thus unclear how to relate the generalized Helmholtz decomposition to the Hodge decomposition.

\subsection{Harmonic and Potential Games}
\label{s:harmonic}

\citet{candogan:11} derive a Hodge decomposition for games that is closely related in spirit to our generalized Helmholtz decomposition -- although the details are quite different. \citet{candogan:11} work with classical games (probability distributions on finite strategy sets). Their losses are multilinear, which is easier than our setting, but they have constrained solution sets, which is harder in many ways. Their approach is based on combinatorial Hodge theory \citep{jiang:11} rather than differential and symplectic geometry. Finding a best-of-both-worlds approach that encompasses both settings is an open problem.

\section{Type Consistency}
\label{s:types}

The next two sections carefully work through the units in classical mechanics and two-player games respectively. The third section briefly describes a use-case for type consistency.

\subsection{Units in Classical Mechanics}
\label{s:units_cm}

Consider the well-known Hamiltonian
\begin{equation}
    \cH(p,q) = \frac{1}{2}\left(\kappa\cdot q^2 + \frac{1}{\mu}\cdot p^2\right)
\end{equation}
where $q$ is position, $p=\mu\cdot \dot{q}$ is momentum, $\mu$ is mass, $\kappa$ is surface tension and $\cH$ measures energy. The units (denoted by $\tau$) are
\begin{equation}
  \begin{matrix}
    \tau(q) = m 
    & & \tau(p)= \frac{kg\cdot m}{s} \\ \\
    \tau(\kappa) = \frac{kg}{s^2}
    & & \tau(\mu) = kg 
  \end{matrix}
\end{equation}
where $m$ is meters, $kg$ is kilograms and $s$ is seconds. Energy is measured in joules, and indeed it is easy to check that $\tau(\cH) = \frac{kg\cdot m^2}{s^2}$.

Note that the units for differentation by $x$ are $\tau(\frac{\dd}{\dd x}) = \frac{1}{\tau(x)}$. For example, differentiating by time has units $\frac{1}{s}$. Hamilton's equations state that $\dot{q} = \frac{\dd H}{\dd p} = \frac{1}{\mu}\cdot p$ and $\dot{p} = -\frac{\dd \cH}{\dd q} = -\kappa \cdot q$ where
\begin{equation}
  \begin{matrix}
    \tau(\dot{q}) = \frac{m}{s} 
    & &
    \tau(\dot{p}) = \frac{kg\cdot m}{s^2}
    \\ \\
    \tau\left(\frac{\dd}{\dd q}\right) = \frac{1}{m}
    & &
    \tau\left(\frac{\dd}{\dd p}\right) = \frac{s}{kg\cdot m}
  \end{matrix}
\end{equation}
The resulting flow describing the dynamics of the system is
\begin{equation}
    \bxi = \dot{q}\cdot \frac{\dd}{\dd q} + \dot{p}\cdot \frac{\dd}{\dd p}
    = \frac{1}{\mu}p\cdot\frac{\dd}{\dd q} -\kappa q\cdot \frac{\dd}{\dd p}
\end{equation}
with units $\tau(\bxi)=\frac{1}{s}$. Hamilton's equations can be reformulated more abstractly via symplectic geometry. Introduce the symplectic form
\begin{equation}
    \omega = dq \wedge dp
    \quad\text{with units}\quad
    \tau(\omega) = \frac{kg\cdot m^2}{s}.
\end{equation}
Observe that contracting the flow with the Hamiltonian obtains
\begin{equation}
    \iota_{\bxi}\omega = \omega(\bxi, \bullet) 
    = d H = \frac{\dd \cH}{\dd q} \cdot dq + \frac{\dd \cH}{\dd p}\cdot dp
\end{equation}
with units $\tau(d \cH) = \tau(\cH) = \frac{kg\cdot m^2}{s^2}$. 

\paragraph{Losses in classical mechanics.}
Although there is no notion of ``loss'' in classical mechanics, it is useful (for the next section) to keep pushing the formal analogy. Define the ``losses''
\begin{equation}
    \label{e:classical_losses}
    \ell_1(q,p) = \frac{1}{\mu}\cdot qp
    \quad\text{and}\quad
    \ell_2(q,p) = -\kappa\cdot qp
\end{equation}
with units $\tau(\ell_1) = \frac{m^2}{s}$ and $\tau(\ell_2) = \frac{kg^2\cdot m^2}{s^3}$. The Hamiltonian dynamics can then be recovered game-theoretically by differentiating $\ell_1$ and $\ell_2$ with respect to $q$ and $p$ respectively. It is easy to check that
\begin{equation}
    \bxi = \frac{\dd \cH}{\dd p}\frac{\dd}{\dd q} 
    - \frac{\dd \cH}{\dd q}\frac{\dd}{\dd p}
    = \frac{\dd\ell_1}{\dd q}\frac{\dd}{\dd q} + \frac{\dd\ell_2}{\dd p}\frac{\dd}{\dd p}.
\end{equation}

\paragraph{The duality between vector fields and differential forms.}
Finally recall that the symplectic form in games was not ``pulled out of thin air'' as $\omega=dq\wedge dp$, but rather derived as $\omega = d\bxi^\flat$, where $\bxi^\flat$ is the differential form corresponding to the vector field $\bxi$ under the musical isomorphism $\flat:TM\rightarrow T^*M$.

It is instructive to compute $\bxi^\flat$ in the case of a classical mechanical system and see what happens. Naively, we would guess that the musical isomorphism is $\left(\frac{\dd }{\dd q}\right)^\flat = dq$ and $\left(\frac{\dd}{\dd p}\right)^\flat = dp$. However, applying the naive musical isomorphism to $\bxi$ to get
\begin{equation}
    \bxi^\flat = \frac{\dd \ell_1}{\dd q}\cdot dq + \frac{\dd \ell_2}{\dd p}\cdot dp
\end{equation}
results in a \emph{type violation} because
\begin{equation}
    \tau\left(\frac{\dd \ell_1}{\dd q}\cdot dq\right) 
    = \tau(\ell_1) = \frac{m^2}{s}
\end{equation}
whereas
\begin{equation}
    \tau\left(\frac{\dd \ell_2}{\dd p}\cdot dp\right) = 
    \tau(\ell_2) = \frac{kg^2\cdot m^2}{s^3}
\end{equation}
and we cannot add objects with different types. 

To correct the type inconsistency, define the musical isomorphism as
\begin{equation}
  \left(\frac{\dd}{\dd q}\right)^\flat=\frac{\mu}{2}\cdot dq
  \quad\text{and}\quad
  \left(\frac{\dd}{\dd p}\right)^\flat=\frac{1}{2\kappa}\cdot dp
\end{equation}
with inverse
\begin{equation}
  \left(dq\right)^\sharp = \frac{2}{\mu}\cdot \frac{\dd}{\dd q}
  \quad\text{and}\quad
  \left(dp\right)^\sharp = 2\kappa\cdot \frac{\dd}{\dd p}.
\end{equation}
The correction terms in the direction $\flat:TM\rightarrow T^*M$ invert the coupling terms $\kappa$ and $\frac{1}{\mu}$ that were originally introduced into the Hamiltonian for physical reasons. Applying the corrected musical isomorphism to $\bxi$ yields
\begin{equation}
  \bxi^\flat 
  = \frac{\mu}{2} \cdot \frac{\dd f}{\dd q}\cdot dq 
  + \frac{1}{2\kappa} \cdot \frac{\dd g}{\dd p}\cdot dp
  = \frac{1}{2}\left(p\cdot dq - q\cdot dp\right).
\end{equation}
The two terms of $\bxi^\flat$ then have coherent types 
\begin{equation}
  \begin{matrix}
    \tau\left(\frac{\dd \ell_1}{\dd q}\cdot \mu\cdot dq\right)
    = \frac{m}{s}\cdot kg\cdot m
    = \frac{kg\cdot m^2}{s}
    \\
    \tau\left(\frac{\dd \ell_2}{\dd p}\cdot \frac{1}{\kappa}\cdot dp\right)
    = \frac{kg\cdot m}{s^2}\cdot\frac{s^2}{kg} \cdot \frac{kg\cdot m}{s}
    = \frac{kg\cdot m^2}{s}
  \end{matrix}
\end{equation}
as required. The associated two form is
\begin{equation}
  \omega := d\bxi^\flat 
  = -\left(\mu \cdot \frac{\dd^2 f}{\dd q\dd p} 
  - \frac{1}{\kappa}\cdot \frac{\dd^2 g}{\dd q \dd p}\right)
  dq\wedge dp \\
  = - dq\wedge dp
\end{equation}
which recovers the symplectic form (up to sign) with units $\tau(\omega)=\frac{kg\cdot m^2}{s}$ as required. Finally, observe that
\begin{align}
  \langle\bxi, \bxi^\flat\rangle 
  & = \frac{1}{2}\left\langle\frac{p}{\mu}\cdot \frac{\dd}{\dd q} -\kappa q\cdot\frac{\dd}{\dd p},p\cdot dq - q\cdot dp\right\rangle \\
  & = \frac{1}{2}\left(\kappa\cdot q^2+\frac{1}{\mu}\cdot p^2\right) = \cH(p,q)
\end{align}
recovering the Hamiltonian.

\subsection{Units in Two-Player Games}
\label{s:units_tp}

Without loss of generality let $\wt=(\x;\y)$ where we refer to $\x$ as position and $\y$ as momentum so that $\tau(\x) = m$ and $\tau(\y) = \frac{kg\cdot m}{s}$. The aim of this section is to check type-consistency under these, rather arbitrarily assigned, units. Since we are considering a game, we do not require that $\x$ and $\y$ have the same dimension -- even though this would necessarily be the case for a physical system. The goal is to verify that units can be consistently assigned to games.

Consider a quadratic two player game of the form
\begin{equation}
  \ell_1(\wt) = \frac{1}{2}\left(\begin{matrix}
    \x^\intercal \,\,\, \y^\intercal
  \end{matrix}\right)\left(\begin{matrix}
    \bA_{11} \,\,\, \bA_{12} \\
    \bA_{21} \,\,\, \bA_{22}
  \end{matrix}\right)\left(\begin{matrix}
    \x \\ \y
  \end{matrix}\right)
  + \left(\begin{matrix}
    \x^\intercal \,\,\, \y^\intercal
  \end{matrix}\right)
  \left(\begin{matrix}
    \bbb_1 \\ \bbb_2
  \end{matrix} \right)
\end{equation}
and
\begin{equation}
  \ell_2(\wt) = \frac{1}{2}\left(\begin{matrix}
    \x^\intercal \,\,\, \y^\intercal
  \end{matrix}\right)\left(\begin{matrix}
    \bC_{11} \,\,\, \bC_{12} \\
    \bC_{21} \,\,\, \bC_{22}
  \end{matrix}\right)\left(\begin{matrix}
    \x \\ \y
  \end{matrix}\right)
  + \left(\begin{matrix}
    \x^\intercal \,\,\, \y^\intercal
  \end{matrix}\right)
  \left(\begin{matrix}
    \bd_1 \\ \bd_2
  \end{matrix} \right)
\end{equation}
We restrict to quadratic games since our methods only involve first and second derivatives. We assume the matrices $\bA$ and $\bC$ are symmetric without loss of generality so that, for example, $\bA_{12}=\bA_{21}^\intercal$. Adding constant terms to $\ell_1$ and $\ell_2$ makes no difference to the analysis so they are omitted.

By \eqref{e:classical_losses}, the units for $\ell_1$ and $\ell_2$ should be $\frac{m^2}{s}$ and $\frac{kg\cdot m^2}{s^3}$ respectively. We can therefore derive the correct units for each of the components of the quadratic losses as
\begin{equation}
   \underbrace{\left(\begin{matrix}
    m \,\,\, \frac{kg\cdot m}{s}
  \end{matrix}\right)
  \left(\begin{matrix}
    \frac{1}{s} &   \frac{1}{kg} \\ \\
    \frac{1}{kg} & \frac{s}{kg^2}
  \end{matrix}\right)
  \left(\begin{matrix}
    m \\ \\ \frac{kg\cdot m}{s}
  \end{matrix}\right)}_{\wt^\intercal\bA\wt}
  + \underbrace{\left(\begin{matrix}
    m \,\,\, \frac{kg\cdot m}{s}
  \end{matrix}\right)
  \left(\begin{matrix}
    \frac{m}{s} \\ \\ \frac{m}{kg}
  \end{matrix} \right)}_{\wt^\intercal \bbb}
\end{equation}
for $\ell_1$ and
\begin{equation}
   \underbrace{\left(\begin{matrix}
    m \,\,\, \frac{kg\cdot m}{s}
  \end{matrix}\right)
  \left(\begin{matrix}
    \frac{kg^2}{s^3} &  \frac{kg}{s^2} \\ \\
    \frac{kg}{s^2}   & \frac{1}{s}
  \end{matrix}\right)
  \left(\begin{matrix}
    m \\ \\ \frac{kg\cdot m}{s}
  \end{matrix}\right)}_{\wt^\intercal\bC\wt}
  + \underbrace{\left(\begin{matrix}
    m \,\,\, \frac{kg\cdot m}{s}
  \end{matrix}\right)
  \left(\begin{matrix}
    \frac{kg^2\cdot m}{s^3} \\ \\ \frac{kg\cdot m}{s^2}
  \end{matrix} \right)}_{\wt^\intercal\bd}
\end{equation}
for $\ell_2$. It follows from a straightforward computation that the vector field $\bxi=\frac{\dd\ell_1}{\dd \x}\frac{\dd}{\dd \x} + \frac{\dd \ell_2}{\dd \y}\frac{\dd}{\dd \y}$ has type $\tau(\bxi)=\frac{1}{s}$ as required.

The presymplectic form $\omega = d\bxi^\flat$ makes use of the musical isomorphism $\flat:T^M\rightarrow T^*M$. As in section~\ref{s:units_cm}, if we naively define $(\frac{\dd}{\dd \x})^\flat = d\x$ and $(\frac{\dd}{\dd \y})^\flat = d\y$ then
\begin{equation}
     \bxi^\flat = \frac{\dd \ell_1}{\dd\x}\cdot d\x + \frac{\dd \ell_2}{\dd \y}\cdot d\y
\end{equation}
which is type inconsistent because $\tau(\frac{\dd\ell_1}{\dd \x}\cdot d\x) = \frac{m^2}{s}$ and $\tau(\frac{\dd \ell_2}{\dd \y}\cdot d\y) = \frac{kg^2\cdot m^2}{s^3}$.

\paragraph{Type-consistency via SVD.}
It is necessary, as in section~\ref{s:units_cm}, to correct the naive musical isomorphism by taking into account the coupling constants for the mixed position-momentum terms. In the classical setup the coupling constants were the scalars $\frac{1}{\mu}$ and $\kappa$, whereas in a game they are the off-diagonal blocks $\bA_{12}$ and $\bC_{12}$. 

Apply singular value decomposition to factorize
\begin{equation}
    \bA_{12} = \bU_{\bA}^\intercal\bD_\bA\bV_\bA
    \quad\text{and}\quad
    \bC_{12} = \bU_{\bC}^\intercal\bD_\bC\bV_\bC
\end{equation}
where the entries of the diagonal matrices have types $\tau(\bD_\bA)=\frac{1}{kg}$ and $\tau(\bD_\bC) = \frac{kg}{s^2}$, and \emph{the types of the orthogonal matrices $\bU$ and $\bV$ are pure scalars}. The diagonal matrices $\bD_\bA$ and $\bD_\bC$ have the same types as $\frac{1}{\mu}$ and $\kappa$ in the classical system since they play the same coupling role.

Extending the procedure adopted in the section~\ref{s:units_cm}, fix the type-inconsistency by defining the musical isomorphisms as
\begin{equation}
  \left(\frac{\dd}{\dd \x}\right)^\flat = \bU_\bA^\intercal\bD_\bA^{-1}\bU_\bA\cdot d\x
\end{equation}
and
\begin{equation}
  \left(\frac{\dd}{\dd \y}\right)^\flat = \bV_\bC^\intercal\bD_\bC^{-1}\bV_\bC\cdot d\y.
\end{equation}
Alternatively, the isomorphisms can be computed by noting that $\bU_\bA^\intercal\bD_\bA^{-1}\bU_\bA = (\sqrt{\bA_{12}\bA_{21}})^{-1}$ and $\bV_\bC^\intercal\bD_\bC^{-1}\bV_\bC = (\sqrt{\bC_{21}\bC_{12}})^{-1}$.

The dual isomorphism $\sharp:T^*M\rightarrow TM$ is then
\begin{equation}
  \left(d\x\right)^\sharp = \bU_\bA^\intercal\bD_\bA\bU_\bA\cdot\frac{\dd}{\dd \x}
  \,\text{ and }\,
  \left(d\y\right)^\sharp = \bV_\bC^\intercal\bD_\bC\bV_\bC\cdot\frac{\dd}{\dd \y}
\end{equation}
If 
\begin{equation}
  \bxi = \left(\begin{matrix}
    \bA_{12}\y + \bbb_1\\
    \bC_{21}\x + \bd_2
  \end{matrix}\right)
\end{equation}
then it follows that 
\begin{equation}
  \bxi^\flat = \left(\begin{matrix}
    \bU_\bA^\intercal\bD_\bA^{-1}\bU_\bA \bbb_1 + \bU_\bA^\intercal\bU_\bA\y
    \\
    \bV_\bC^\intercal\bD_\bC^{-1}\bV_\bC \bd_2 + \bV_\bC^\intercal\bV_\bC\x
  \end{matrix}\right)
\end{equation}
with associated closed two form
\begin{equation}
  \omega_\tau = d\bxi^\flat = -\left(\bU_\bA^\intercal\bV_\bA - \bU_\bC^\intercal\bV_\bC\right)d\x\wedge d\y.
\end{equation}
where the notation $\omega_\tau$ emphasizes that the two-form is type-consistent.

\subsection{What Does Type-Consistency Buy?}

\begin{eg}\label{eg:types}
  Consider the loss functions
  \begin{equation}
    f(x,y) = xy
    \quad\text{and}\quad
    g(x,y) = 2xy,
  \end{equation}
  with $\bxi = (y,2x)$. There is no function $\phi:\bR^2\rightarrow \bR$ such that $\grad \phi = \bxi$. However, there is a family of functions $\phi_\alpha(x,y) = \alpha\cdot xy$ which satisfies
  \begin{equation}
    \big\langle\bxi,\grad \phi_\alpha\big\rangle
    = \alpha \cdot(x^2 + 2y^2)\geq0
    \quad\text{for all $\alpha>0$}.
  \end{equation}
\end{eg}
Although $\bxi$ is not a potential field, there is a family of functions on which $\bxi$ performs gradient descent -- albeit with coordinate-wise learning rates that may not be optimal. The vector field $\bxi$ arguably does not require adjustment. This kind of situation often arises when the learning rates of different parameters are set adaptively during training of neural nets, by rescaling them by positive numbers. 

The vanilla and type-consistent 1-forms corresponding to $\bxi$ are, respectively,
\begin{equation}
  \bxi^\flat = y\cdot dx + 2x\cdot dy
  \quad\text{and}\quad
  \bxi^\flat_\tau = y\cdot dx + x\cdot dy
\end{equation}
with 
\begin{equation}
  \omega = d\bxi^\flat_\text{non} = dx\wedge dy
  \quad\text{and}\quad
  \omega_\tau = d\bxi^\flat_\tau = 0.
\end{equation}
It follows that the \emph{type-consistent} symplectic gradient adjustment is zero. Type-consistency `detects' that no gradient adjustment is needed in example~\ref{eg:types}.

\end{document}